\documentclass{article} 
\usepackage{tencent_ailab_tech_report}

\usepackage[T1]{fontenc} 
\usepackage[utf8]{inputenc} 

\usepackage{microtype}
\usepackage{hyperref}
\usepackage{url}
\usepackage{booktabs}
\definecolor{darkblue}{rgb}{0, 0, 0.5}
\hypersetup{colorlinks=true, citecolor=darkblue, linkcolor=darkblue, urlcolor=darkblue}
\usepackage{wrapfig}
\usepackage{amsthm}
\usepackage{amsmath,amssymb,amsfonts}

\newcommand{\PP}{\mathbb{P}}
\newcommand{\R}{\mathbb{R}}
\newcommand{\E}{\mathbb{E}}
\newcommand{\Var}{\text{Var}}
\newcommand{\op}{\mathrm{op}}

\usepackage{graphicx}
\usepackage{enumitem}
\usepackage{pgfplots}
\usepackage{tabularx}
\usepackage{caption}
\usepackage{subcaption}
\usepackage{bbm}
\usepackage{url}
\usepackage{amsmath} 
\usepackage{pifont}
\usepackage{arydshln}
\usepackage{makecell}
\usepackage{booktabs}
\usepackage{multirow} 
\usepackage{algorithmic}
\usepackage{amsmath}
\usepackage{mathabx} 
\usepackage{enumitem}
\usepackage[ruled,vlined]{algorithm2e}

\usepackage[most]{tcolorbox}
\usepackage{minted}

\newtheorem{theorem}{Theorem}[section]
\newtheorem{lemma}[theorem]{Lemma}

\newtheorem{assumption}[theorem]{Assumption}

\DeclareUnicodeCharacter{211D}{\ensuremath{\mathbb{R}}}
\DeclareUnicodeCharacter{2124}{\ensuremath{\mathbb{Z}}}
\DeclareUnicodeCharacter{21A6}{\ensuremath{\mapsto}}

\DeclareUnicodeCharacter{2264}{\ensuremath{\leq}}
\DeclareUnicodeCharacter{2265}{\ensuremath{\geq}}
\DeclareUnicodeCharacter{2192}{\ensuremath{\rightarrow}}
\DeclareUnicodeCharacter{2194}{\ensuremath{\leftrightarrow}}
\DeclareUnicodeCharacter{2227}{\ensuremath{\wedge}}
\DeclareUnicodeCharacter{22A2}{\ensuremath{\vdash}}
\DeclareUnicodeCharacter{2080}{\ensuremath{_0}}
\DeclareUnicodeCharacter{2081}{\ensuremath{_1}}
\DeclareUnicodeCharacter{2082}{\ensuremath{_2}}
\DeclareUnicodeCharacter{2083}{\ensuremath{_3}}
\DeclareUnicodeCharacter{2084}{\ensuremath{_4}}
\DeclareUnicodeCharacter{2085}{\ensuremath{_5}}
\DeclareUnicodeCharacter{2086}{\ensuremath{_6}}
\DeclareUnicodeCharacter{2087}{\ensuremath{_7}}
\DeclareUnicodeCharacter{2088}{\ensuremath{_8}}
\DeclareUnicodeCharacter{2089}{\ensuremath{_9}}
\DeclareUnicodeCharacter{2200}{\ensuremath{\forall}}
\DeclareUnicodeCharacter{2203}{\ensuremath{\exists}}
\DeclareUnicodeCharacter{2228}{\ensuremath{\lor}}
\DeclareUnicodeCharacter{2260}{\ensuremath{\neq}}
\DeclareUnicodeCharacter{230B}{\ensuremath{\rfloor}} 
\DeclareUnicodeCharacter{2223}{\ensuremath{\mid}}    
\DeclareUnicodeCharacter{230A}{\ensuremath{\lfloor}} 
\DeclareUnicodeCharacter{2090}{\ensuremath{_a}}  
\DeclareUnicodeCharacter{2091}{\ensuremath{_e}}  
\DeclareUnicodeCharacter{2092}{\ensuremath{_o}}  
\DeclareUnicodeCharacter{2093}{\ensuremath{_x}}  
\DeclareUnicodeCharacter{2211}{\ensuremath{\sum}}        
\DeclareUnicodeCharacter{2115}{\ensuremath{\mathbb{N}}}  

\DeclareUnicodeCharacter{2208}{\ensuremath{\in}} 

\newtcblisting{promptbox}[1][]{%
  enhanced,
  breakable,
  colback=pink!10,
  colframe=pink!80,
  arc=1mm,
  boxsep=4pt,
  fontupper=\ttfamily\small,
  title=#1,
  listing only,
  listing engine=listings,
  listing options={
    inputencoding=latin1,    
    breaklines=true,
    breakatwhitespace=false,
    basicstyle=\ttfamily\small,
    columns=fullflexible,
    escapeinside={(*@}{@*)}      
  }%
}

\title{CDE: Curiosity-Driven Exploration for Efficient Reinforcement Learning in Large Language Models}

\author{Runpeng Dai$^{1,3}$$^\dagger$, Linfeng Song$^1$$^\dagger$, Haolin Liu $^{1,4}$, Zhenwen Liang$^1$, Dian Yu$^1$, Haitao Mi$^1$, \\
\textbf{Zhaopeng Tu$^2$, Rui Liu$^{1,5}$, Tong Zheng$^{1,5}$, Hongtu Zhu$^3$, Dong Yu$^1$}
\vspace{1em}\\ 
$^1$Tencent AI Lab,~~
$^2$Tencent Multimodal Department,\\
$^3$University of North Carolina at Chapel Hill,\\
$^4$University of Virginia,\\
$^5$University of Maryland, College Park
\\
$\dagger$ Core contributors 
\\
\texttt{runpeng@unc.edu}, \texttt{lfsong@global.tencent.com}
}

\colmfinalcopy 
\begin{document}

\maketitle

\begin{abstract}
Reinforcement Learning with Verifiable Rewards (RLVR) is a powerful paradigm for enhancing the reasoning ability of Large Language Models (LLMs). Yet current RLVR methods often explore poorly, leading to premature convergence and entropy collapse. 
To address this challenge, we introduce \textbf{Curiosity-Driven Exploration (CDE)}, a framework that leverages the model's own intrinsic sense of curiosity to guide exploration. We formalize curiosity with signals from both the actor and the critic: for the actor, we use perplexity over its generated response, and for the critic, we use the variance of value estimates from a multi-head architecture. Both signals serve as an exploration bonus within the RLVR framework to guide the model. Our theoretical analysis shows that the actor-wise bonus inherently penalizes overconfident errors and promotes diversity among correct responses; moreover, we connect the critic-wise bonus to the well-established count-based exploration bonus in RL. Empirically, our method achieves an approximate \textbf{+3} point improvement over standard RLVR using GRPO/PPO on AIME benchmarks. Further analysis identifies a \textbf{calibration collapse} mechanism within RLVR, shedding light on common LLM failure modes.
\end{abstract}

\section{Introduction}
The reasoning ability of Large Language Models (LLMs) has achieved remarkable performance across diverse application domains such as mathematics \citep{shao2024deepseekmath} and coding \citep{guo2024deepseek}. A central challenge in this development is how to efficiently elicit high-quality Chain-of-Thought (CoT) reasoning. A major breakthrough earlier this year was the introduction of reinforcement learning with verifiable rewards (RLVR), a training paradigm in which models are optimized directly using the signal of final-answer correctness. This approach removes the burden of designing and training potentially fragile reward models. Despite the emergence of various RLVR training algorithms, such as GRPO \citep{guo2024deepseek} and DAPO \citep{yu2025dapo}, key issues remain. In particular, problems such as premature convergence and phenomena like entropy collapse \citep{cui2025entropy} have been widely observed during training, posing fundamental challenges to the stability and effectiveness of RLVR.

These challenges stem from the classic exploration-exploitation dilemma in reinforcement learning \citep{SuttonRL}. Phenomena like entropy collapse reveal a critical flaw in the training process: it is heavily biased towards exploitation, causing models to converge prematurely instead of sufficiently exploring their environment for better solutions. Although the RL literature encompasses a wide range of exploration strategies, these methods exhibit significant limitations when applied to LLMs. Simple heuristics, including entropy bonuses \citep{haarnoja2018soft} and $\epsilon$-greedy policies \citep{SuttonRL}, either injecting randomness to the environment or encouraging the policy to be more stochastic. These methods are either provably suboptimal in theory \citep{dann2022guarantees} or demonstrate debatable effectiveness in complex environments like Deep RL \citep{andrychowicz2021matters} and LLM-based reasoning \citep{cui2025entropy, shen2025entropy}. More principled methods are count-based, which incentivize visiting rarely explored state–action pairs. Algorithms such as UCB \citep{lai1987adaptive} for multi-armed bandits, LinUCB \citep{li2010contextual}, LSVI-UCB \citep{jin2020provably} for linear bandits/MDPs achieve near-optimal exploration guarantees across a variety of settings. However, these methods (i) require computationally intensive operations such as matrix inversion, and (ii) heavily depend on highly expressive representations of state-action pair (reasoning paths), which become impractical for reasoning-focused LLMs with long chains of thought. Thus, developing efficient and scalable exploration methods for LLMs remains a key open challenge.

\begin{wrapfigure}{r}{0.3\textwidth}
    \vspace{-15pt}
    \centering
    \includegraphics[width=\linewidth]{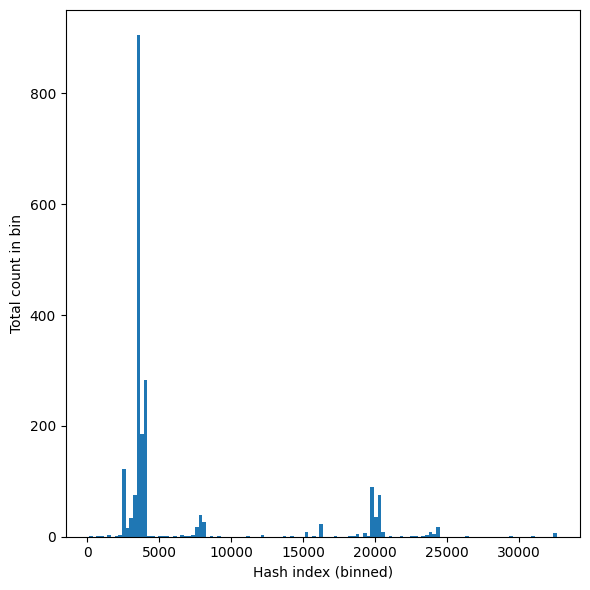}
    \caption{Distribution of number of visitations across hash-cells.}
    \vspace{-20pt}
    \label{fig:pie}
\end{wrapfigure}

In our preliminary experiments, we investigate the direct application of count-based exploration methods to the RLVR setting. To avoid the computational burden of matrix inversion, we adopt the SimHash technique \citep{tang2017exploration}, which maps the embedding of a CoT response into a discrete hash code, and then uses the visitation frequency of hash cells as pseudo-counts (see Section~\ref{sec:count} for details). However, as illustrated in Figure~\ref{fig:pie}, this approach proves problematic: it is difficult to meaningfully characterize a complex CoT reasoning trajectory with a fixed embedding vector. In practice, most responses collapse into the same or neighboring hash grids, leading to a highly concentrated distribution of counts and thus undermining the effectiveness of count-based exploration for RLVR.

In this work, we propose an intuitive approach that leverages the model's intrinsic sense of \textbf{curiosity} as a guide for exploration. An LLM, having been trained on vast reasoning corpora, develops a sophisticated internal model of what constitutes a familiar versus a novel reasoning pattern. This parallels early childhood development~\citep{chu2020play}, where learning is not driven by a external summary and count of experiences, but is instead propelled by an intrinsic curiosity to explore novel situations. We formalize this principle in our \textbf{Curiosity-Driven Exploration (CDE)} framework, which considers curiosity signals from both the actor and the critic. For the actor, perplexity (PPL) over its generated response serves as the curiosity measure. For the critic, we measure curiosity via the variance of its posterior value distribution. We then approximate this posterior by extending the PPO framework with a multi-head, bootstrapped structure. The curiosity signals are served as an exploration bonus, shaping the reward and advantage functions to effectively guide exploration.

Our theoretical analysis offers further insights into the properties of our method. (i) Theorem \ref{thm:calibration} interprets the proposed perplexity-based bonus, showing that it intrinsically penalizes overconfident errors while encouraging diversity among correct responses. (ii) Theorem \ref{thm:consistency} establishes that in the linear MDP setting, our critic-based exploration bonus is theoretically equivalent to classical count-based bonuses, grounding our approach in established exploration principles.

Our empirical evaluation demonstrates consistent performance gains across four widely used mathematics benchmarks (AIME25, AIME24, AMC23, and MATH), including an approximate \textbf{+3 point} improvement on the challenging AIME benchmarks. Furthermore, our analysis of the training process supports our theoretical findings and reveals a phenomenon we term \textbf{calibration collapse}: under a naive GRPO policy, the model's confidence progressively decouples from its correctness, while adding PPL bonus mitigates this miscalculation.

\section{Preliminaries: RLVR, GRPO and PPO}
We formulate the language generation process of LLMs as a sequential decision-making problem \citep{yu2025dapo,yue2025vapo}. Specifically, we consider two reinforcement learning algorithms: \emph{Group Relative Policy Optimization} (GRPO), a critic-free method, and \emph{Proximal Policy Optimization} (PPO), a canonical actor–critic method. We adopt the training paradigm of Reinforcement Learning with Verifiable Rewards (RLVR) \citep{guo2025deepseek,lambert2024tulu} and utilize a rule-based verifier to compare the generated response with the ground truth to judge its correctness.

\subsection{Group Relative Policy Optimization (GRPO, \citealt{shao2024deepseekmath})}

GRPO is an REINFORCE-style optimization algorithm. Let $\pi_{\theta}$ denote the LLM policy with parameters $\theta$. At each training step, given a prompt $q$ sampled from the dataset $\mathcal{D}$, the current policy $\pi_{\theta}$ generates a group of $G$ candidate outputs $\{o_1, o_2, \ldots, o_G\}$. For each candidate $o_i$, we compute its total reward $r_i = r(o_i, q)$.

The advantage for each output is computed by normalizing its reward with respect to the group's rewards:
\begin{equation*}
    A_{i} = \frac{r_i - \operatorname{mean}(r_1, \ldots, r_G)}{\operatorname{std}(r_1, \ldots, r_G) + \delta},
\end{equation*}
where $\delta$ is a small constant for numerical stability. The same advantage $A_i$ is applied to all tokens in $o_i$. Let $\pi_{\theta_{\text{old}}}$ be the policy from the previous step and $\pi_{\text{ref}}$ the original pre-trained model. GRPO maximizes:
\begin{equation*}
    \mathcal{L}_{\text{GRPO}}(\theta) = \mathbb{E}_{q \sim \mathcal{D}, \{o_i\} \sim \pi_{\theta_{\text{old}}}}
    \left[ \frac{1}{G} \sum_{i=1}^G \frac{1}{|o_i|} \sum_{t=1}^{|o_i|} \mathcal{L}_{\theta}(\tilde{r}_{i,t}, A_i) \right]
    - \beta D_{\text{KL}}\left(\pi_{\theta} \| \pi_{\text{ref}}\right),
\end{equation*}
where the clipped objective is
\begin{equation*}
    \mathcal{L}_{\theta}(\tilde{r}_{i,t}, A_i) =
    \min\left( \tilde{r}_{i,t} A_i,
    \text{clip}(\tilde{r}_{i,t}, 1-\varepsilon, 1+\varepsilon) A_i \right), \quad
    \tilde{r}_{i,t} = \frac{\pi_{\theta}(o_{i,t} \mid q, o_{i,<t})}{\pi_{\theta_{\text{old}}}(o_{i,t} \mid q, o_{i,<t})}.
\end{equation*}
Here, $\varepsilon$ and $\beta$ control the ratio clipping threshold and the KL-penalty strength, respectively. The clipping mitigates large, unstable policy updates, while the KL term constrains deviation from $\pi_{\text{ref}}$.

\subsection{Proximal Policy Optimization (PPO, \citealt{schulman2017proximal})}

PPO is an actor–critic algorithm that maintains both a policy (actor) $\pi_{\theta}$ and a value function (critic) $V_{\phi}$ with parameters $\phi$, estimating the expected total reward from a given state (prompt and sequence prefix). The advantage function in PPO leverages the critic to reduce variance. Specifically, \textbf{Generalized Advantage Estimation (GAE)} is applied to compute token-level advantages. For an output $o_i$ with sentence-level reward $r_i$, the GAE at token $t$ is:
\begin{equation*}
    A_{i,t} = \sum_{l=t}^{|o_i|} (\gamma\lambda)^{l-t} \delta_{i,l},
\end{equation*}
where
\begin{equation*}
    \delta_{i,l} = r_{i,l} + \gamma V_{\phi}(q, o_{i, \leq l+1}) - V_{\phi}(q, o_{i, \leq l}),
\end{equation*}
and in our setting $r_{i,l} = 0$ for all non-terminal tokens, with $r_{i,|o_i|} = r_i$. The hyperparameters $\gamma$ and $\lambda$ are the discount factor and GAE trace-decay, respectively. The PPO objective is:
\begin{equation*}
    \mathcal{L}_{\text{PPO}}(\theta, \phi) =
    \mathbb{E}_{q \sim \mathcal{D}, \{o_i\} \sim \pi_{\theta_{\text{old}}}}
    \Big[ \frac{1}{|o_i|} \sum_{t=1}^{|o_i|}  \left[\mathcal{L}_{\theta}(\tilde{r}_{i,t}, A_{i,t}) - c_1 \mathcal{L}_{\phi}(q,o_{i,<t}, r_i) \right]\Big]
    - \beta D_{\text{KL}}\left(\pi_{\theta} \| \pi_{\text{ref}}\right),
\end{equation*}
where $\mathcal{L}_{\theta}$ is as in GRPO but with per-token $A_{i,t}$, and the value loss is:
\begin{equation*}
    \mathcal{L}_{\phi}(\phi) = \left( V_{\phi}(q, o_{i,<t}) - r_i \right)^2.
\end{equation*}
In practice, we alternate optimization of the actor ($\theta$) and the critic ($\phi$).

\section{CDE: Curiosity-Driven Exploration}
\label{sec:method}

In this section, we first explore count-based exploration and identify two of its key challenges. To overcome these challenges, we propose Curiosity-Driven Exploration (CDE), a systematic framework that considers curiosity signals from both the actor and the critic. We introduce the detailed formulations of actor and critic curiosity in Section \ref{sec:actor} and Section \ref{sec:critic}, respectively.

\subsection{Challenge of Count-based exploration for RLVR}
\label{sec:count}
The core idea of count-based exploration is to measure the occurrence of Chain-of-Thought (CoT) patterns via sentence embeddings, and to assign an exploration bonus to rarely occurred CoTs. While conceptually appealing, this approach faces two major challenges:

\begin{itemize}[left=0pt]
    \item \textbf{Curse of dimensionality:} Classical count-based methods (e.g., LSVI-UCB \citep{jin2020provably}, CFPO \citep{cassel2024warm}) rely on computing the inverse of the covariance matrix $\Lambda_t^{-1}$ (see Appendix \ref{sec:LinMDP}) to construct historical visitation ellipsoids. For high-dimensional embeddings, this operation is computationally prohibitive.
\end{itemize}

To circumvent the need for matrix inversion, we investigated hash-based counts \citep{tang2017exploration} (details in Appendix \ref{sec:hash}), which project sentence embeddings into discrete hash grids and treat grid visitation frequency as a proxy for counts. However, this alternative introduces a second limitation:

\begin{itemize}[left=0pt]
    \item \textbf{Poor expressiveness of embeddings:} As illustrated in Figure \ref{fig:pie}, after hash coding, most CoT embeddings collapse into neighboring hash grids. This clustering highlights the limited ability of sentence embeddings to distinguish between diverse reasoning patterns, leading to ineffective exploration.
\end{itemize}

In this work, we move beyond the paradigm of explicit state–action counts and instead utilize the model's own measure of novelty. Our approach is motivated by the key intuition that agents, much like children in early cognitive development \citep{chu2020play}, exhibit a form of \textit{curiosity}. They respond with confidence when revisiting familiar states or CoT patterns, yet display uncertainty and exploratory behavior when confronted with novel situations. This learning process is not driven by a external count of experiences but is instead propelled by an intrinsic drive to explore. We first detail the implementation of the actor's curiosity before turning to the critic's.

\subsection{Exploration Guided by Actor Curiosity}
\label{sec:actor}

We model actor curiosity as the actor's uncertainty about its own actions. Intuitively, a response that is surprising to the actor—i.e., has a low probability under its current policy—likely resides in an underexplored region of its learned distribution.

A natural and computationally efficient measure of this surprise is the perplexity of the actor's generation. We formalize this as a sentence-level curiosity bonus, defined as the negative average log-probability of a generated sentence $o=\{o_1, \ldots, o_T\}$, given a prompt $q$:
\begin{equation}
    B_{actor}(q,o) = -\frac{1}{T} \sum_{t=1}^T \log \pi(o_t| o_{<t}, q)
    \label{eq:actor_bonus}
\end{equation}
where $\pi$ denotes the actor policy. A higher value for $B_{actor}(q,o)$ indicates greater surprise and thus a stronger intrinsic reward signal for exploration. 

However, practically simply adding this bonus to the original reward can be unstable and sub-optimal. Unconstrained exploration might incentivize the model to generate high-perplexity but low-quality or inaccurate responses (a behavior known as reward hacking), or lead to over-exploration where the policy fails to converge to a stable, high-quality output. To ensure that exploration remains tethered to the primary objective of maximizing the original reward signal, we integrate the bonus using an adaptive clipping mechanism. The total sentence-level reward, $\tilde{r}$, is a combination of the original reward signal $r(q,o)$ and the curiosity bonus $B_{actor}(q,o)$, where the bonus is capped relative to the original reward:
\begin{equation}
    \tilde{r}(q, o) = r(q,o) + {\color{red} \omega_t} \min(\frac{|r(q,o)|}{\color{blue}\kappa}, {\color{green}\alpha} B_{actor}(q,o))
    \label{eq:actor_eq}
\end{equation}

This formulation promotes exploration by rewarding sentences that the actor finds surprising, while constraining the bonus to remain a fraction of the original reward. In this way, the model is discouraged from trading response quality for novelty. The behavior of this reward function is controlled by three key hyperparameters:

\begin{itemize}[leftmargin=*]
    \item The \textbf{bonus weight} $\omega_t$ is a dynamic coefficient, typically set with an annealing schedule to decrease over the course of training. This allows for more aggressive exploration in the early stages and then gradually shifts focus towards exploitation of high-reward regions as the policy converges.
    
    \item The \textbf{clipping ratio} $\kappa$ governs the maximum size of the curiosity bonus relative to the original reward. By capping the bonus at $|r(q,o)|/\kappa$, it ensures the bonus remains a supplement and prevents it from dominating the learning signal. This is particularly crucial when $r(q,o)$ is negative, as it guarantees the bonus cannot reverse the sign of the reward, maintaining the integrity of the penalty.
    
    \item The \textbf{bonus scaling factor} $\alpha$ normalizes the curiosity bonus $B_{\text{actor}}(q,o)$ before it is compared to the clipped reward. A higher $\alpha$ allows the curiosity bonus to reach the clipping threshold more easily, whereas a smaller $\alpha$ diminishes its potential impact.
\end{itemize}

\paragraph{Intuitions and Theoretical Foundation}

While the above formulation specifies how the perplexity bonus is shaped and controlled through its hyperparameters, it is equally important to understand its qualitative effect on model behavior. To build this intuition, we analyze responses along two axes: correctness and actor perplexity. Among these four categories, two require particular attention:

\begin{minipage}{0.6\textwidth}
\begin{enumerate}[left=0pt]
    \item Incorrect responses with low PPL indicate that the model is highly confident in its answer, yet the response is wrong. This reflects overfitting and should be penalized.
    \item Correct responses with high PPL suggest that the model is less familiar with such answers, but they nevertheless turn out to be successful. This reflects effective exploration and should be encouraged.
\end{enumerate}
\end{minipage}\hfill
\begin{minipage}{0.35\textwidth}
    \centering
    \includegraphics[width=0.8\linewidth]{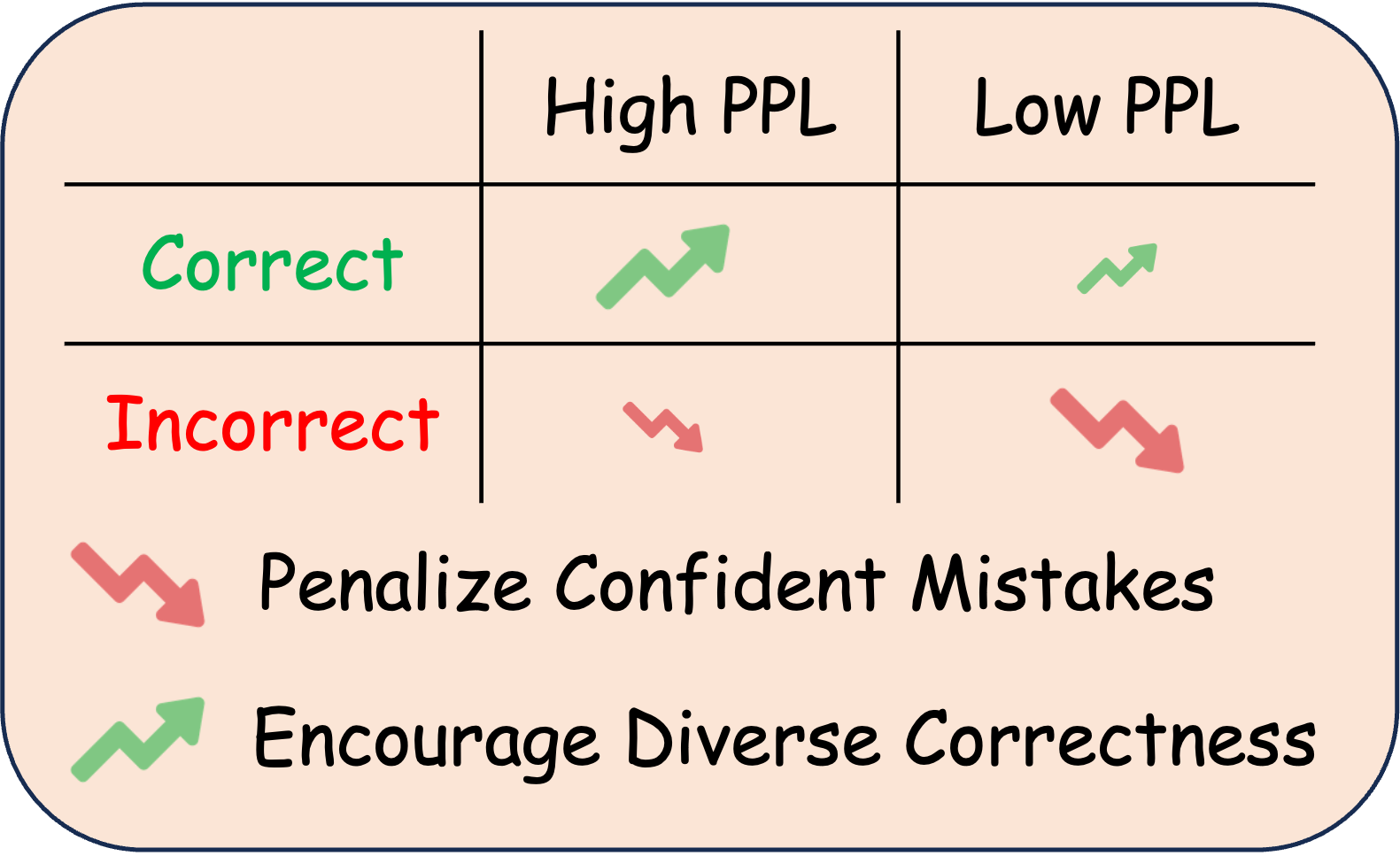}
    \captionof{figure}{Responses by correctness and avg PPL.}
    \label{fig:axis}
\end{minipage}

As illustrated in Figure~\ref{fig:axis}, we find out that the PPL bonus intrinsically penalizes confident mistakes while encouraging novel correct responses. For correct responses, those novel responses (with higher PPL) receive a larger positive reward. For incorrect responses, those confident responses (with lower PPL) receive larger penalty as it receives smaller PPL bonus. The following theorem formalizes this intuition; its precise statement and proof are deferred to Appendix~\ref{sec:calibration}.

\begin{theorem}
\label{thm:calibration}
Let $\pi_t$ denote the policy at training step t. With PPL bonus in Equation \eqref{eq:actor_bonus}, the update to $\pi_{t+1}$ calibrates the policy's confidence as follows:

(i) Among correct responses, trajectories with higher perplexity receive a larger relative probability increase.

(ii) Among incorrect responses, trajectories with lower perplexity receive a larger relative probability decrease.
\end{theorem}

\begin{minipage}{0.35\textwidth}
    \centering
    \includegraphics[width=0.8\linewidth]{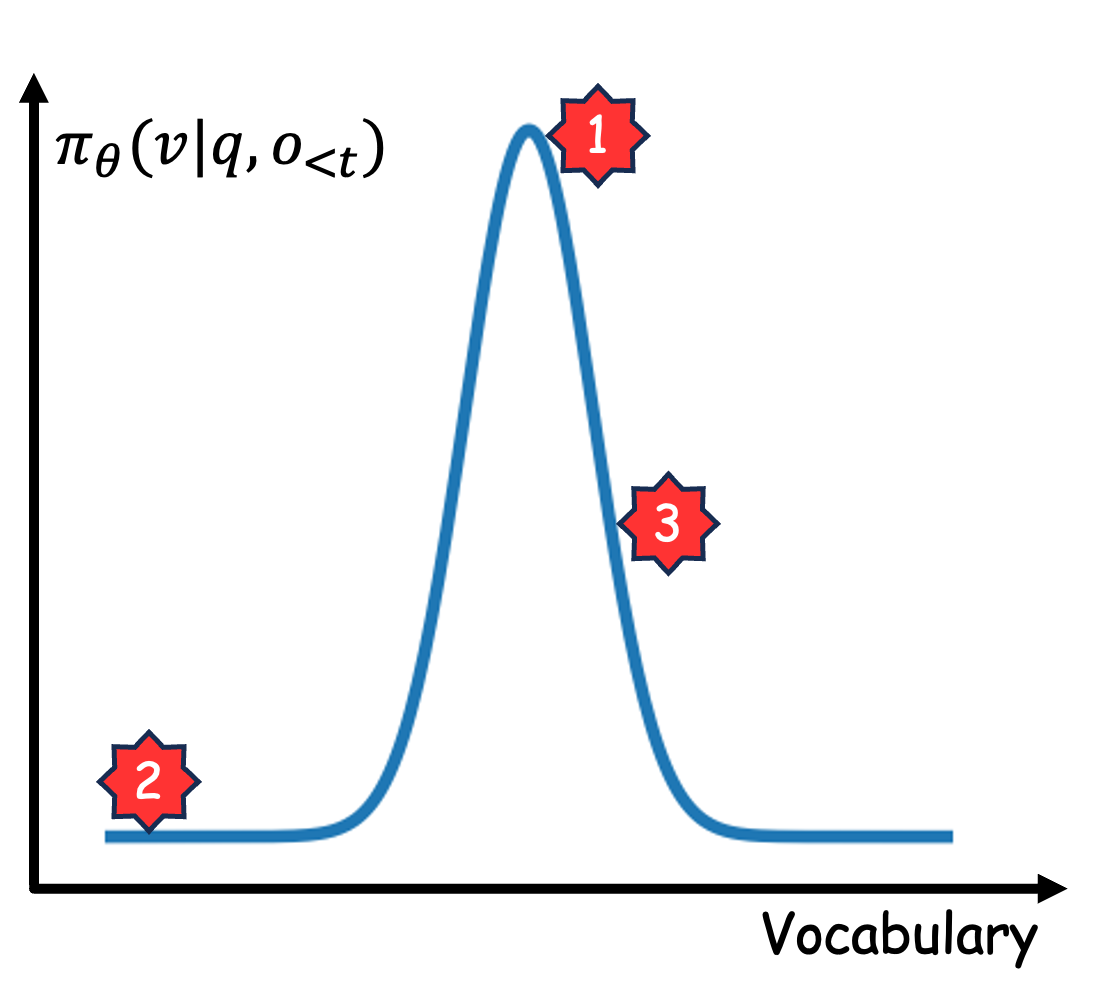}
    \captionof{figure}{An illustration of the model's policy distribution for selecting the next token.}
    \label{fig:entropy_fig}
\end{minipage}\hfill
\begin{minipage}{0.65\textwidth}
Previous analyses distinguish the PPL bonus from the entropy bonus, which is sample-agnostic at the token level. The entropy at any given step depends solely on the policy's probability distribution and is independent of the token ultimately sampled. Because the calculation considers the entire next-token distribution $\pi_{\theta} \left(v \mid q, o_{<t}\right)$ (Equation \ref{eq:Entropy}), the bonus $\mathcal{H}_t$ remains constant for any potential outcome (Figure \ref{fig:entropy_fig}). Therefore, even when the model makes a high-confidence error by sampling token 1, the entropy bonus fails to penalize that choice.
\begin{equation}
    \mathcal{H}_t = - \sum_{v \in \mathcal{V}} \pi_{\theta} \left(v \mid q, o_{<t}\right)\log \pi_{\theta} \left(v \mid q, o_{<t}\right).
    \label{eq:Entropy}
    \end{equation}
\end{minipage}

\subsection{Exploration Guided by Critic Curiosity}
\label{sec:critic}
In contrast to critic-free methods such as REINFORCE and GRPO, the critic (value function) in actor–critic frameworks provides a higher-level understanding of the prompt–response pair by estimating the expected reward-to-go. Since this estimate is learned directly from collected trajectories, its posterior distribution conditioned on the observed data naturally reflects the degree of coverage: regions with dense data yield concentrated (low-variance) posteriors, whereas sparsely sampled regions result in higher uncertainty. Posterior distributions are a well-established means of quantifying predictive uncertainty in deep learning models \citep{gal2016dropout, lakshminarayanan2017simple}. As shown in Figure~\ref{fig:dist}, the orange curve exhibits lower variance—evidence of better data coverage—whereas the other curve is more dispersed, reflecting greater uncertainty.

\begin{wrapfigure}{r}{0.3\textwidth}
    \vspace{-10pt}
    \centering
    \includegraphics[width=\linewidth]{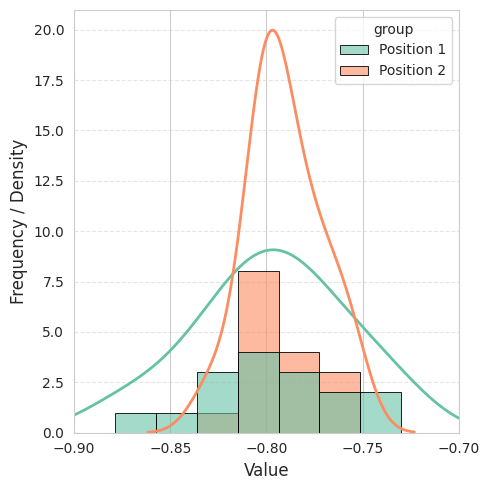}
    \caption{An illustration of two posterior distributions of the critic and their bootstrap approximations.}
    \vspace{-15pt}
    \label{fig:dist}
\end{wrapfigure}

To approximate the posterior distribution of value estimates, we adopt the classical bootstrap method \citep{davison1997bootstrap}, widely used in statistics and increasingly recognized in the RL community as an effective tool for exploration \citep{osband2016deep, ciosek2019better, bai2021principled}. We implement this idea through a multi-head critic (upper-left subfigure in Figure~\ref{fig:MHA_dynamic}), where $K$ critics $\{\widehat{V}_1, \ldots, \widehat{V}_K\}$ share a common LLM backbone. Each head is trained on a resampled subset of the collected trajectories (bottom subfigure in Figure~\ref{fig:MHA_dynamic}), thereby producing an empirical approximation to the posterior distribution.  

We then use the standard deviation across the $K$ heads as a principled curiosity signal, guiding the policy toward regions of high disagreement where the value function remains uncertain and under-explored. In the following theorem, we establish a surprising yet intuitive result: under a Linear MDP assumption, the standard deviation of the bootstrap critics is a consistent estimator of the pseudo-count bonus.

\begin{theorem}
In linear MDPs, the standard deviation across multi-head critics can serve as a consistent estimator for the pseudo-count exploration bonus, $\sqrt{\phi_{n,h}^\top \Lambda_{n,h}^{-1} \phi_{n,h}}$, as used in LSVI-UCB \citep{jin2020provably} and CFPO \citep{cassel2024warm}, where $\phi_{n,h} = \phi(s_{n,h}, a_{n,h})$ is the feature vector of a state-action pair and $\Lambda_{n,h} = \sum_{i=0}^n \phi_{i,h} \phi_{i,h}^\top + \lambda I$ is the coverage matrix.
\label{thm:consistency}
\end{theorem}

The rigorous formulation of the linear MDP assumptions and the proof of Theorem~\ref{thm:consistency} are provided in Appendix~\ref{sec:LinMDP}, while empirical results in Section~\ref{sec:MHC_analysis} further support this finding. Building on this foundation, we now describe the training procedure of the multi-head PPO algorithm, which follows the standard stages of vanilla PPO: (i) generating trajectories with the actor, (ii) updating the actor, and (iii) updating the critic. The key distinction is that we incorporate the multi-head variance as an exploration bonus, encouraging the policy to visit under-explored regions. A visual illustration of these steps is shown in Figure~\ref{fig:MHC}.

\begin{figure*}[t!]
    \centering
    \includegraphics[width=0.9\linewidth]{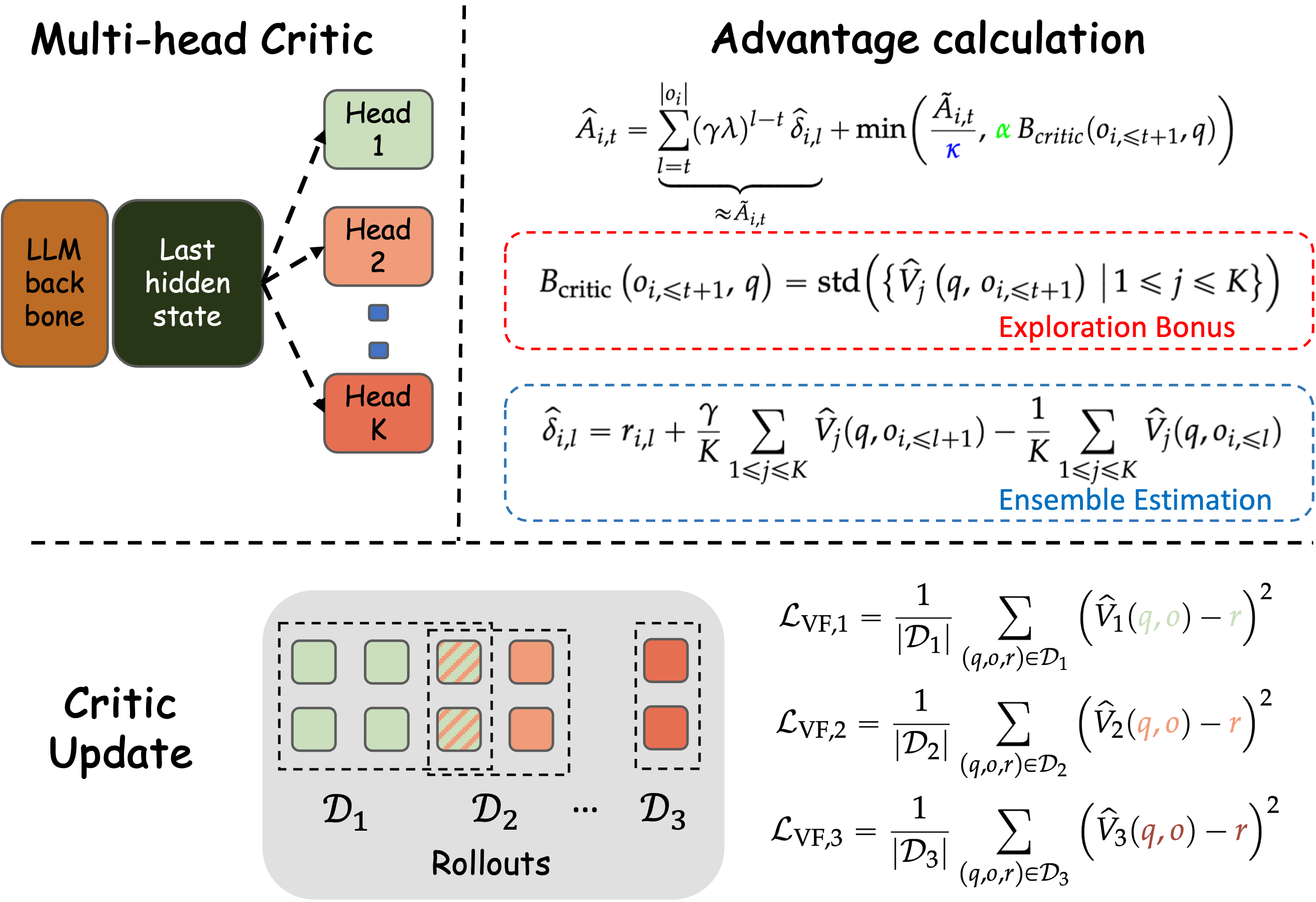}
    \caption{Illustration of the multi-head critic framework.}
    \label{fig:MHC}
\end{figure*}

\begin{itemize}
    \item \textbf{Actor roll-out:}  
        Given a prompt $q$, the actor generates a set of responses $\{o_1, \ldots, o_n\}$. Each response is denoted as $o_i = \{o_{i,1}, \ldots, o_{i,|o_i|}\}$. Correspondingly, we associate each response with a verifiable reward $r_i$. For clarity, we focus on the case of a single prompt $q$.  
    
    \item \textbf{Actor update:}  
        In this step, the advantage is estimated as  
        \begin{equation}
        \widehat{A}_{i,t} =  
        \underbrace{\sum_{l=t}^{|o_i|} (\gamma\lambda)^{l-t} \widehat{\delta}_{i,l}}_{\approx \tilde{A}_{i,t}}
        + {\color{red}\omega_t}\min\left(\frac{|\tilde{A}_{i,t}|}{\color{blue}\kappa},{\color{green}\alpha}B_{\text{critic}}(q, o_{i,\leq t+1})\right).
        \label{eq:rlvr_critic}
        \end{equation}
        
        The advantage consists of two components. The first term, $\tilde{A}_{i,t}$, largely follows the standard advantage estimation in PPO, except that we exploit bootstrap estimators by using an \emph{ensemble} of value functions rather than a single point estimate:  
        \begin{equation*}
        \widehat{\delta}_{i,l} = r_{i,l} 
        + \frac{\gamma}{K}\sum_{j=1}^K \widehat{V}_j(q, o_{i, \leq l+1})
        - \frac{1}{K}\sum_{j=1}^K \widehat{V}_j(q, o_{i, \leq l}).
        \end{equation*}
        
        The second term of Equation \ref{eq:rlvr_critic} introduces the \emph{multi-head critic bonus} ($B_{\text{critic}}$), governed by the bonus weight $\omega_t$, clipping ratio $\kappa$, and scaling factor $\alpha$ (see discussion following Equation~\eqref{eq:actor_eq} for interpretation). Specifically, $B_{\text{critic}}$ is defined as the standard deviation across the $K$ value heads, encouraging exploration by assigning higher bonus to actions leading to uncertain/less-visited regions:  
        \begin{equation}
        B_{\text{critic}}\left(q, o_{i,\leq t+1}\right) 
        = \text{std}\left( \big\{ \widehat{V}_{j}(q, o_{i,\leq t+1}) \big| 1 \leq j \leq K \big\} \right).
        \end{equation}

    \item \textbf{Critic update:}  
        We use the collected roll-outs to update the critic. For notational convenience, let the dataset be  
        \begin{equation}
            \mathcal{D} = \{(q, o_{i,\leq t}, r_i) | i \in [n], t \in [|o_i|]\},
            \label{eq:dataset}
        \end{equation}
        consisting of (prompt, partial response, reward) triplets. For each critic head $j$, we sample without replacement a subset $\mathcal{D}_j \subset \mathcal{D}$ of size $|\mathcal{D}_j| = \zeta |\mathcal{D}|$, where the hyperparameter $\zeta \in (0,1]$ controls the fraction of data assigned per head. Smaller $\zeta$ increases head diversity, while larger $\zeta$ improves sample efficiency. The multi-head critic is then updated with the following bootstrap loss:  
        \begin{equation*}
        \mathcal{L}_{\phi} 
        = \frac{1}{\zeta K |\mathcal{D}|}
           \sum_{j=1}^K 
           \sum_{(q,o,r) \in \mathcal{D}_j}
           \left(\widehat{V}_j(q,o) - r \right)^2 .
        \end{equation*}
        
\end{itemize}

\section{Experiments}
\subsection{Dataset and Model}
In this paper, we adopt DAPO-17K \citep{yu2025dapo} for training and evaluate the performance of CDE on four challenging mathematical reasoning benchmarks: MATH \citep{hendrycks2021measuring}, AMC23 \citep{amc}, AIME24, and AIME25 \citep{aime}. These evaluations are designed to assess CDE’s effectiveness in comparison to standard PPO and GRPO algorithms. Due to computational resource constraints, we conduct training with a reduced setting. All experiments are implemented within the Verl framework using the Qwen3-4B-Base model \citep{yang2025qwen3}. For fair comparison, all the models use the default prompt in DAPO-17K as shown in Appendix \ref{sec:prompt} and the implementation details in Appendix \ref{sec:detail} to further elaborate on the training settings.

\subsection{Main Results}
The main results are presented in Table~\ref{tab:main} while the training dynamic is presented in Figure \ref{fig:dynamics}. Here PPL bonus denote adding Curiosity bonus on actors as in Equation \ref{eq:actor_eq}, $K$ Heads represents multi-head critic PPO with $K$ head critics. We report both average Pass@1 accuracy and Pass@16 results on evaluation datasets. The key observations are as follows:

\begin{table*}[h!]
    \centering
    \resizebox{\textwidth}{!}{%
    \begin{tabular}{l c cc cc cc c}
        \toprule
        \multirow{2}{*}{\bf Model} 
        & \multicolumn{1}{c}{\bf MATH} 
        & \multicolumn{2}{c}{\bf AMC23} 
        & \multicolumn{2}{c}{\bf AIME24} 
        & \multicolumn{2}{c}{\bf AIME25} 
        & \multicolumn{1}{c}{\bf Avg} 
        \\
        & Avg@1 & Avg@16 & Pass@16 & Avg@16 & Pass@16 & Avg@16 & Pass@16\\
        \midrule
        Qwen3-4B-Base & 23.1     & 10.9    & 53.8    & 1.5 & 8.4 & 1.3 & 8.3 & 9.2   \\
        \midrule
        \multicolumn{8}{c}{\textbf{\textit{GRPO based methods}}} \\
        Qwen3-4B-Base-GRPO   &  87.3 & 63.6 & 89.1 & 20.8 & 41.9 &  21.0 & 39.2 & 48.2  \\
        $\drsh$ \textbf{w/ PPL bonus} &  87.7 & 67.8 & 89.5 & 23.3 & 48.5 & 23.5 & 42.5 & 50.6 \\
        \midrule
        \multicolumn{8}{c}{\textbf{\textit{PPO based methods}}} \\
        Qwen3-4B-Base-PPO & 86.6 & 64.1 & 87.2 & 17.8 & 36.0 & 17.5 & 33.7 & 46.5 \\
        $\drsh$ \textbf{w/ PPL bonus} &  87.9 & 66.1 & 88.5 & 18.3 & 37.6 & 18.3 & 33.5 & 47.7 \\
        $\drsh$ \textbf{w/ 2 Heads} & 83.2 & 63.6 & 89.9 & 19.6 & 34.8 & 19.6 & 36.1 & 46.6 \\
        $\drsh$ \textbf{w/ 4 Heads} & 87.3 & 63.9 & 87.9 & 21.5 & 35.5 & 21.5 & 45.5 & 48.5 \\
        $\drsh$ \textbf{w/ 8 Heads} & 85.1 & 66.7 & 86.9 & 21.7 & 46.4 & 19.0 & 37.1 & 48.1 \\
        $\drsh$ \textbf{w/ 16 Heads} & 88.3 & 65.0 & 88.7 & 20.5 & 41.9 & 20.0 & 38.8 & 48.6 \\
        \bottomrule
    \end{tabular}
    }
    \caption{Zero-shot accuracy of different models on the validation datasets. Avg@16 denotes the mean Pass@1 accuracy over 16 sampled generations, while \textbf{Avg} column represents the overall average across datasets, computed as Avg@1 for MATH and Avg@16 for the remaining datasets.}
    \label{tab:main}
\end{table*}

\begin{figure}[h!]
    \centering
    \begin{subfigure}[t]{0.49\columnwidth}
        \centering
        \includegraphics[width=\linewidth]{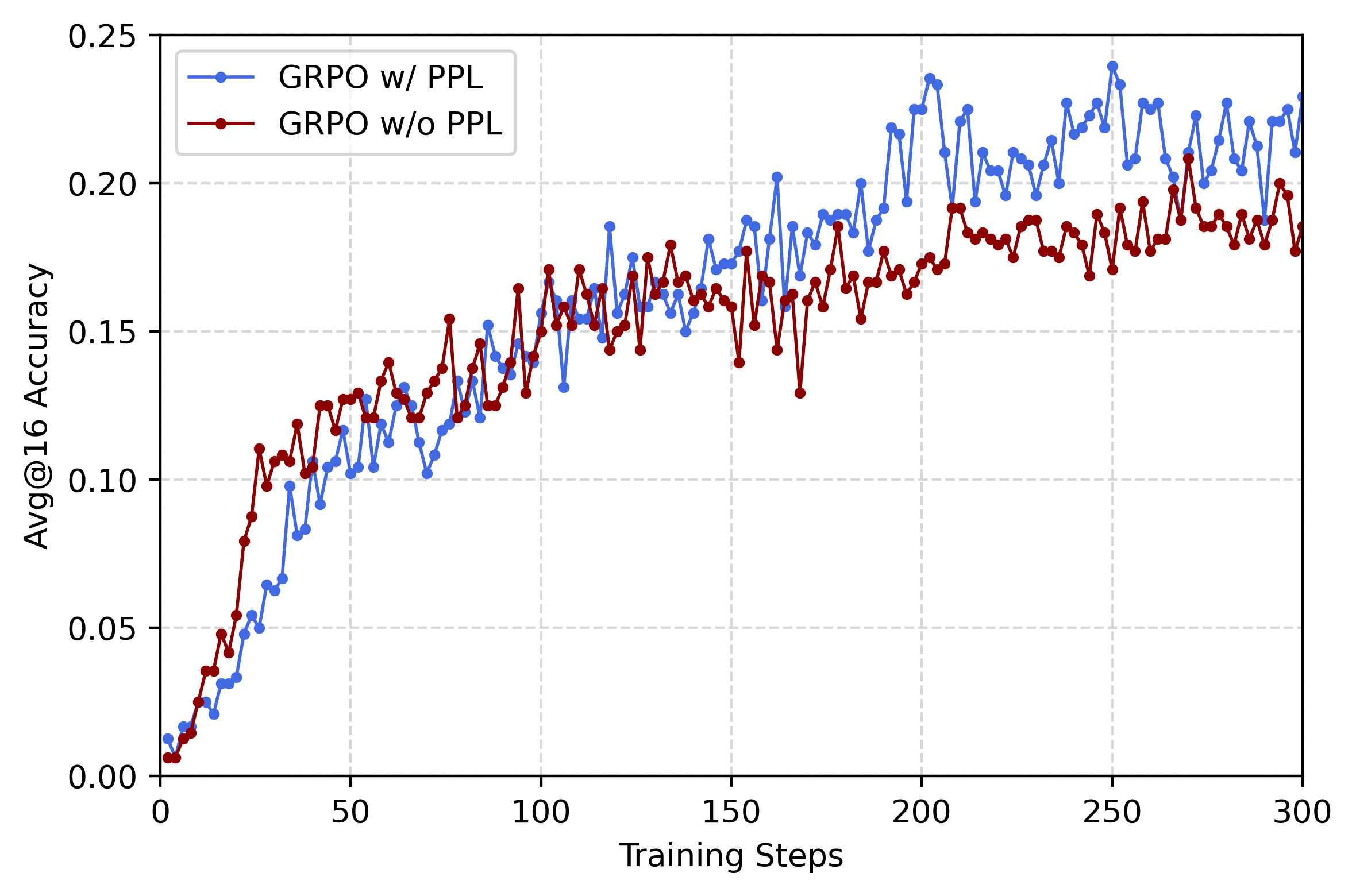}
    \end{subfigure}
    \hfill
    \begin{subfigure}[t]{0.49\columnwidth}
        \centering
        \includegraphics[width=\linewidth]{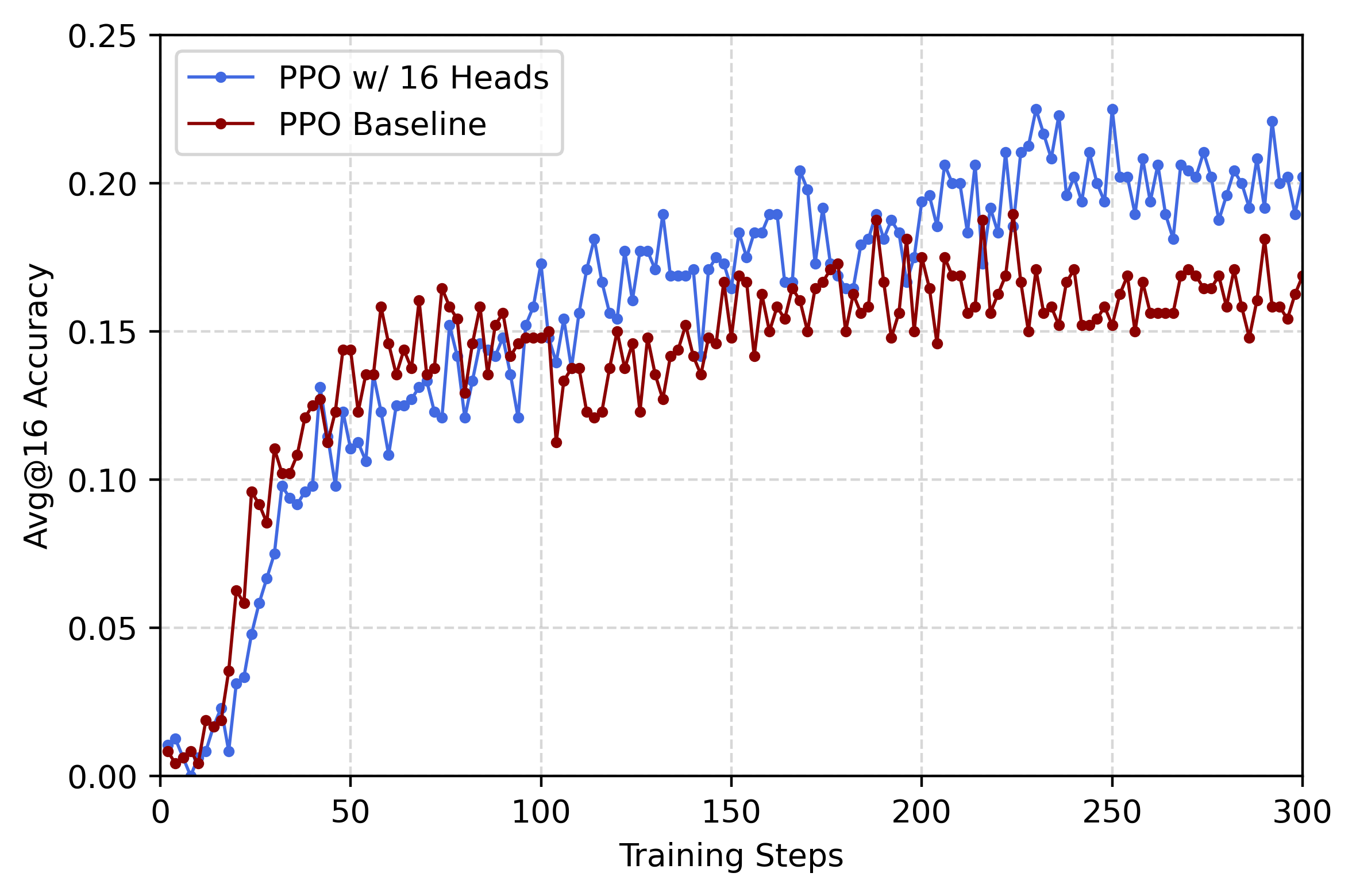}
    \end{subfigure}
    \hfill
    \caption{Comparison of Avg@16 accuracy on AIME25 over training of vanilla GRPO and PPO (Baseline methods) and GRPO with PPL bonus and 16 head multi-head PPO(Our methods).}
    \label{fig:dynamics}
\end{figure}

\begin{itemize}[leftmargin=*]
\item The PPL bonus further enhances the mathematical reasoning ability of the GRPO method, yielding an average improvement of approximately $+2.4$ points across datasets and demonstrating consistent superiority. In particular, our method achieves notable gains on Pass@16, surpassing the baseline GRPO by about $+8$ points on the AIME24 dataset.
\item Across benchmarks, multi-head PPO consistently outperforms vanilla PPO. Using $K=4$ and $K=16$ heads yields average gains of roughly $+2$ points, and we observe an around $+10$ points of increase in Pass@16 on AIME datasets in many cases. 
\item The performance of multi-head PPO generally increases with the number of heads $K$: with $K=2$ delivers negligible gains over the baseline, and performance increase begin to plateau once $K\ge4$, which suggests that a modest number of heads already captures most of the epistemic uncertainty needed.
\item As shown in Figure~\ref{fig:dynamics}, GRPO with PPL bonus and multi-head PPO increase test accuracy more slowly than baseline PPO/GRPO early in training, then catch up and ultimately surpass them. This pattern is consistent with enhanced exploration: the PPL bonus and head disagreement discourage premature exploitation of spurious high-reward trajectories. As state-action coverage expands, these signals calibrate, enabling a smoother shift to targeted exploitation and yielding higher final accuracy. 
\end{itemize}

\subsection{Understanding the Effect of the PPL Bonus}
In this subsection, we present additional experiments to investigate the role of the PPL bonus, from which we derive the following key findings.

\begin{wrapfigure}{r}{0.35\textwidth}
    \vspace{-20pt} 
    \centering
    \includegraphics[width=\linewidth]{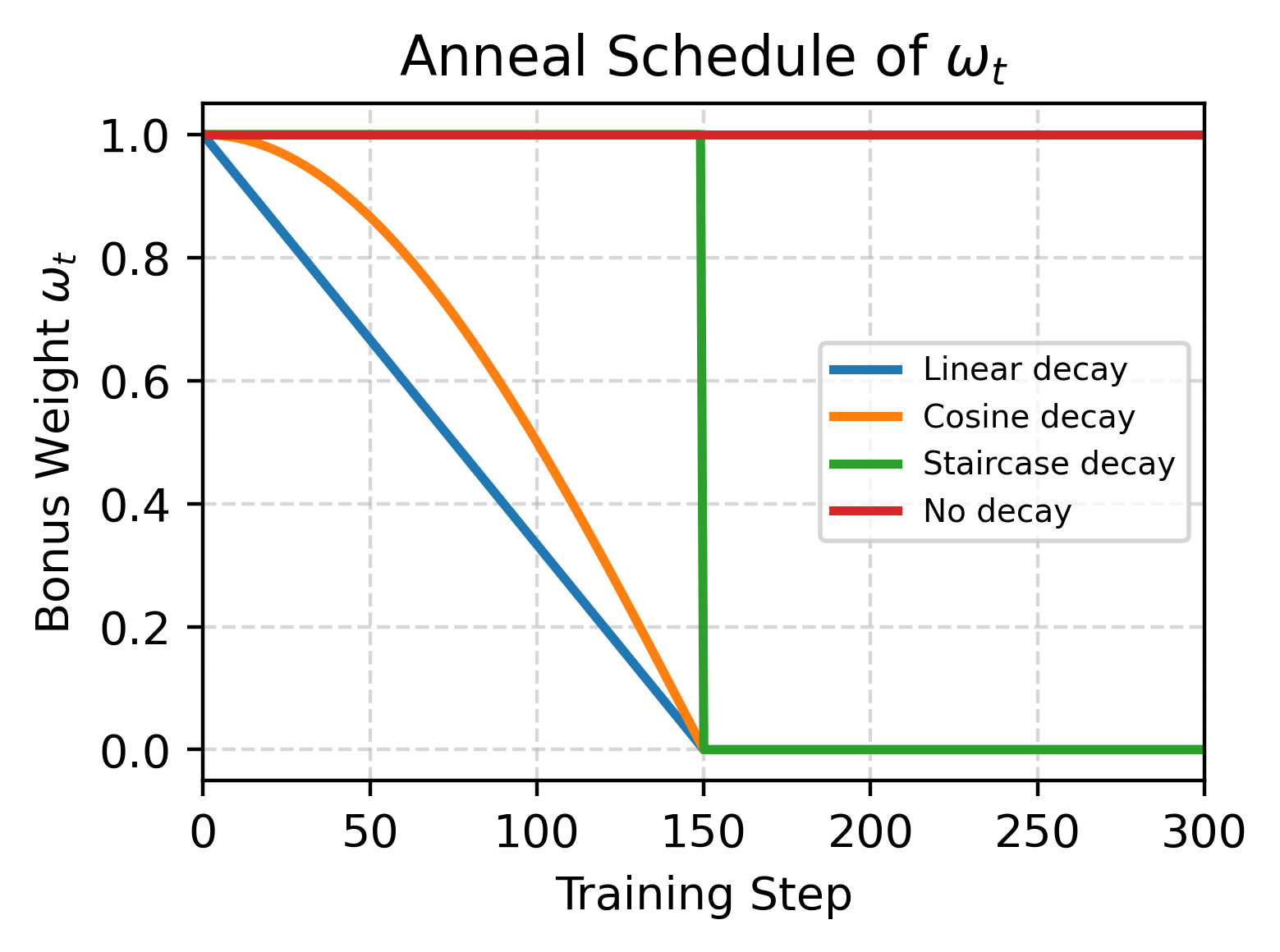}
    \vspace{-0.3in}
    \caption{An illustration of different weight anneal schedules.}
    \label{fig:anneal}
    \vspace{-15pt} 
\end{wrapfigure}

\textbf{Bonus weight decay is crucial}
We compare four schedules for the bonus weight $\omega_t$—\emph{No decay}, \emph{Linear}, \emph{Cosine}, and \emph{Staircase}—as illustrated in Figure~\ref{fig:anneal}, with the performance of models trained under each schedule summarized in Table~\ref{tab:anneal}. Briefly, the \emph{No decay} schedule maintains strong exploration throughout training, while the \emph{Staircase} schedule reduces $\omega_t$ abruptly, enabling strong exploration in the early phase and then removing the bonus for final convergence. The \emph{Linear} and \emph{Cosine} schedules provide intermediate behaviors.

The results in Table~\ref{tab:anneal} underscore two insights: First, decay of the bonus weight is necessary, as all decay schedules outperform the no-decay baseline by enabling a gradual shift from exploration to exploitation. Second, strong exploration in the early phase is crucial, with the staircase scheme proving most effective by sustaining high exploration initially to broaden state–action coverage and then removing the bonus abruptly to allow stable convergence, whereas the gentler cosine and linear decays weaken the signal too soon and thus yield smaller gains.

\begin{table*}[h!]
    \centering
    \resizebox{\textwidth}{!}{%
    \begin{tabular}{l c cc cc cc c}
        \toprule
        \multirow{2}{*}{\bf Model} 
        & \multicolumn{1}{c}{\bf MATH} 
        & \multicolumn{2}{c}{\bf AMC23} 
        & \multicolumn{2}{c}{\bf AIME24} 
        & \multicolumn{2}{c}{\bf AIME25} 
        & \multicolumn{1}{c}{\bf Avg} 
        \\
        & Avg@1 & Avg@16 & Pass@16 & Avg@16 & Pass@16 & Avg@16 & Pass@16\\
        \midrule
        \multicolumn{8}{c}{\textbf{\textit{Bonus Weight Decay Schedules}}} \\
    Qwen3-4B-Base-GRPO      & 87.3 & 63.6 & 91.1 & 21.0 & 41.9 & 20.8 & 39.2 & 48.2 \\
    $\drsh$ \textbf{$\omega_t$ No decay}               & 85.1 & 64.5 & 84.6 & 20.8 & 39.0 & 22.3 & 36.2 & 48.2 \\
    $\drsh$ \textbf{$\omega_t$ Linear decay}           & 85.4 & 66.1 & 91.9 & 23.3 & 40.4 & 20.0 & 40.4 & 48.7 \\
    $\drsh$ \textbf{$\omega_t$ Cosine decay}           & 86.7 & 68.1 & 90.0 & 22.5 & 44.9 & 21.5 & 40.7 & 49.7 \\
    $\drsh$ \textbf{$\omega_t$ Staircase decay}      & 87.7 & 67.8 & 89.2 & 23.5 & 48.5 & 23.3 & 40.3 & 50.6 \\
    \bottomrule
    \end{tabular}
    }
    \caption{Zero-shot accuracy of GRPO models under different PPL bonus weight decay schedules. The schedules follow those illustrated in Figure~\ref{fig:anneal}.}
    \label{tab:anneal}
\end{table*}

\begin{wrapfigure}{r}{0.5\textwidth}
    \vspace{-1pt} 
    \centering
    \includegraphics[width=0.9\linewidth]{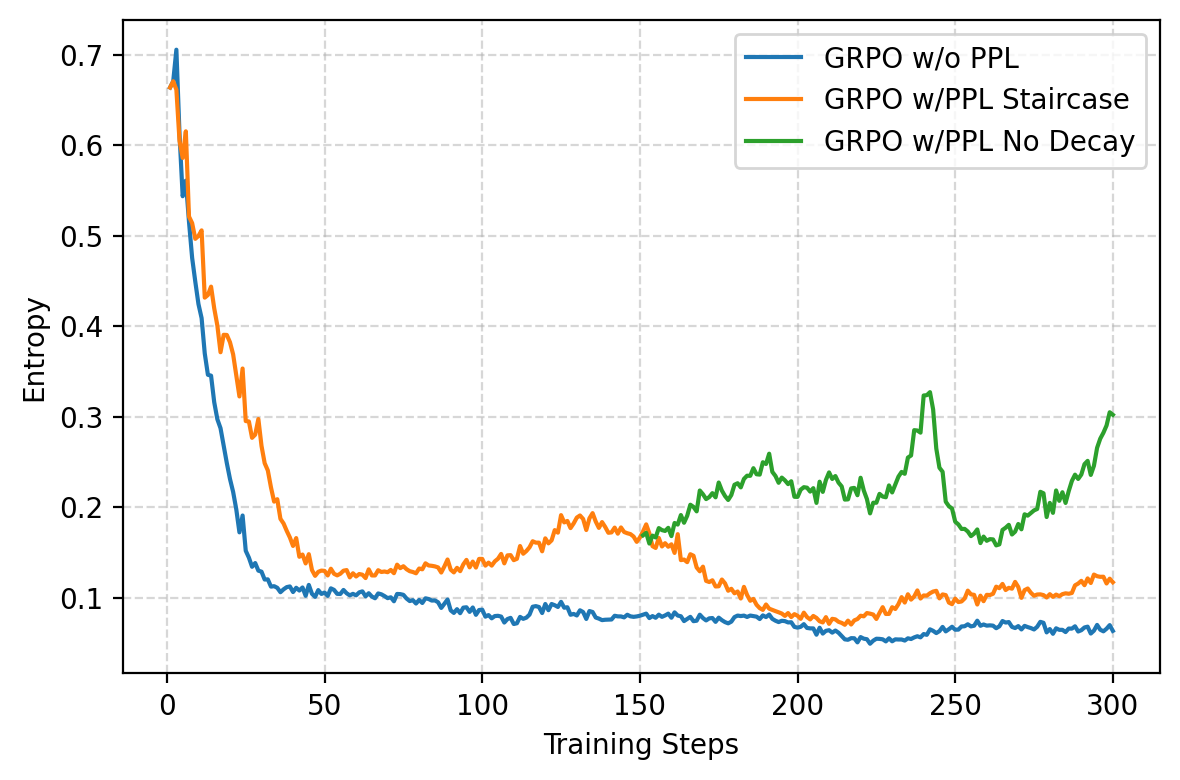}
    \caption{Dynamics of policy entropy over the training process. The bonus weight decay mechanism follows Figure~\ref{fig:anneal}. The Staircase and No Decay schedules share the same early training phase.}
    \label{fig:entropy}
\end{wrapfigure}

\textbf{Analysis of Entropy Dynamics}
As highlighted in prior work, entropy provides an important lens for understanding exploration ability \citep{cui2025entropy}, where a sharp decline in entropy often signals premature convergence and insufficient exploration. Figure~\ref{fig:entropy} illustrates the entropy dynamics of baseline GRPO compared to our proposed methods. First, relative to the baseline, the PPL bonus alleviates entropy collapse, demonstrating its role in promoting exploration. Second, when comparing decay schemes, PPL with No Decay shows persistent fluctuations and fails to converge, whereas Staircase decay yields more stable entropy trajectories. This observation is consistent with our earlier findings that decaying the bonus weight is essential for ensuring stable convergence while still supporting effective exploration.

\begin{figure}[h!]
    \centering
    \begin{subfigure}[t]{0.49\columnwidth}
        \centering
        \includegraphics[width=\linewidth]{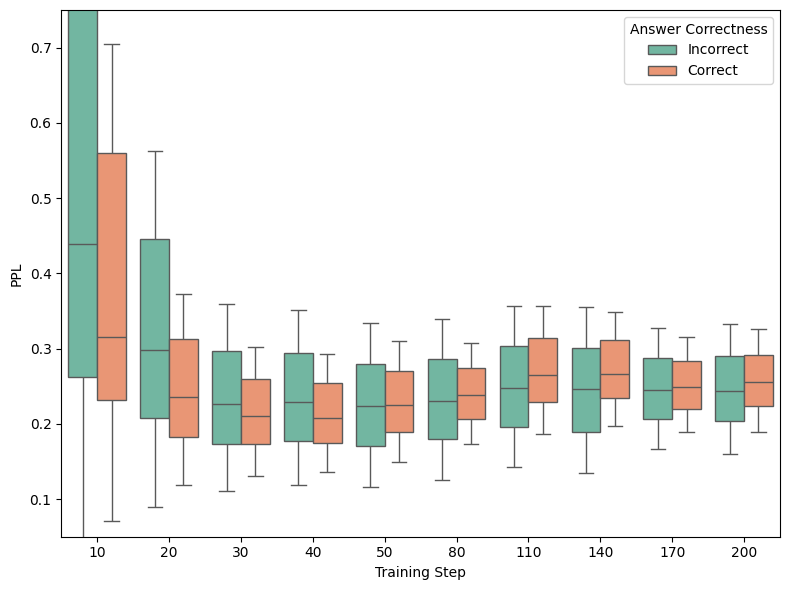}
        \caption{GRPO without PPL bonus}
    \end{subfigure}
    \hfill
    \begin{subfigure}[t]{0.49\columnwidth}
        \centering
        \includegraphics[width=\linewidth]{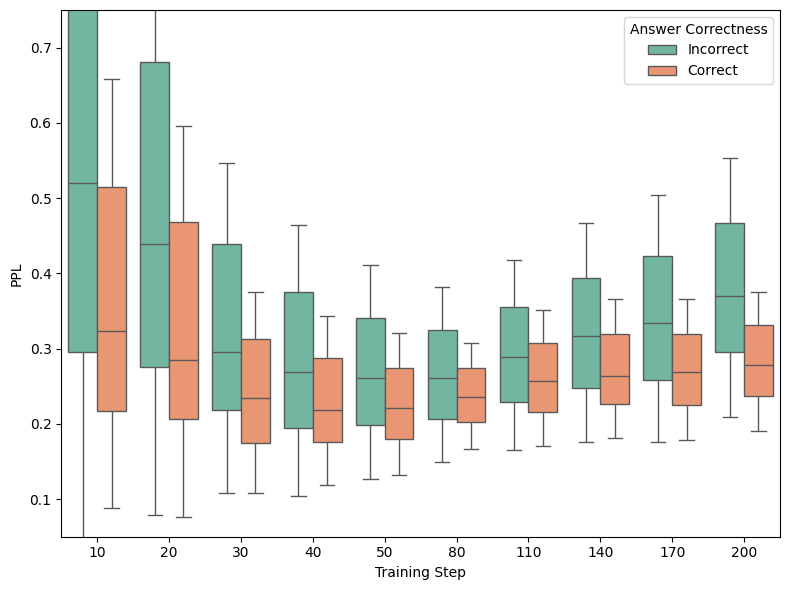}
        \caption{GRPO with PPL bonus}
    \end{subfigure}
    \hfill
    \caption{Average response PPL per training step, stratified by correctness.}
    \label{fig:PPLdynamics}
\end{figure}
\textbf{Analysis of Calibration}
As shown in Figure~\ref{fig:PPLdynamics}, we plot the batch-wise mean response perplexity (PPL), stratified by answer correctness. In subfigure~(a), we observe a phenomenon we term \textbf{calibration collapse}: early in naive GRPO training, correct responses have lower PPL (higher confidence) than incorrect ones, but as training progresses this gap shrinks and ultimately vanishes—confidence no longer tracks correctness. By contrast, with a PPL bonus (subfigure~(b)), this separation is sustained throughout training.

This pattern is explained by Theorem~\ref{thm:calibration}: while both naive GRPO and GRPO with a PPL bonus tend to increase confidence on correct answers, the PPL bonus additionally suppresses confident errors (low-PPL incorrect trajectories), thereby improving calibration.

This finding is original and practically important. Ideally, a trained model should be faithful—confident when its answer is correct and cautious when it is not. Better calibration enhances interpretability and supports inference-time selection strategies such as self-certainty BoN \citep{wang2022self} and DeepConf \citep{fu2025deep}. It also connects to the growing literature on calibrating LLMs, both during training (e.g., \citep{shen2024thermometer}) and at test time (e.g., \citep{ulmer2024calibrating}).

\subsection{Further Analysis of the Multi-Head Critic}
\label{sec:MHC_analysis}

\textbf{Analysis of Dynamics of $B_{\text{critic}}$}
We further examine the dynamics of the multi-head exploration bonus $B_{\text{critic}}$ by tracking its average value over the course of training. Specifically, for each training step, given the roll-outs $\mathcal{D}$ defined in Equation~\ref{eq:dataset}, we compute the average $B_{\text{critic}}$ across (prompt, partial response, reward) triplets within $\mathcal{D}$. As shown in sub-figure (a) of Figure~\ref{fig:MHA_dynamic}, this average decreases steadily as training progresses. The decline reflects that, with more training, similar trajectories are revisited more frequently, leading to reduced disagreement among critic heads. This phenomenon provides empirical support for interpreting the multi-head bonus as analogous to count-based exploration measures. 

\begin{figure}[h!]
    \centering
    \begin{subfigure}[c]{0.6\columnwidth}
        \centering
        \includegraphics[width=0.9\linewidth]{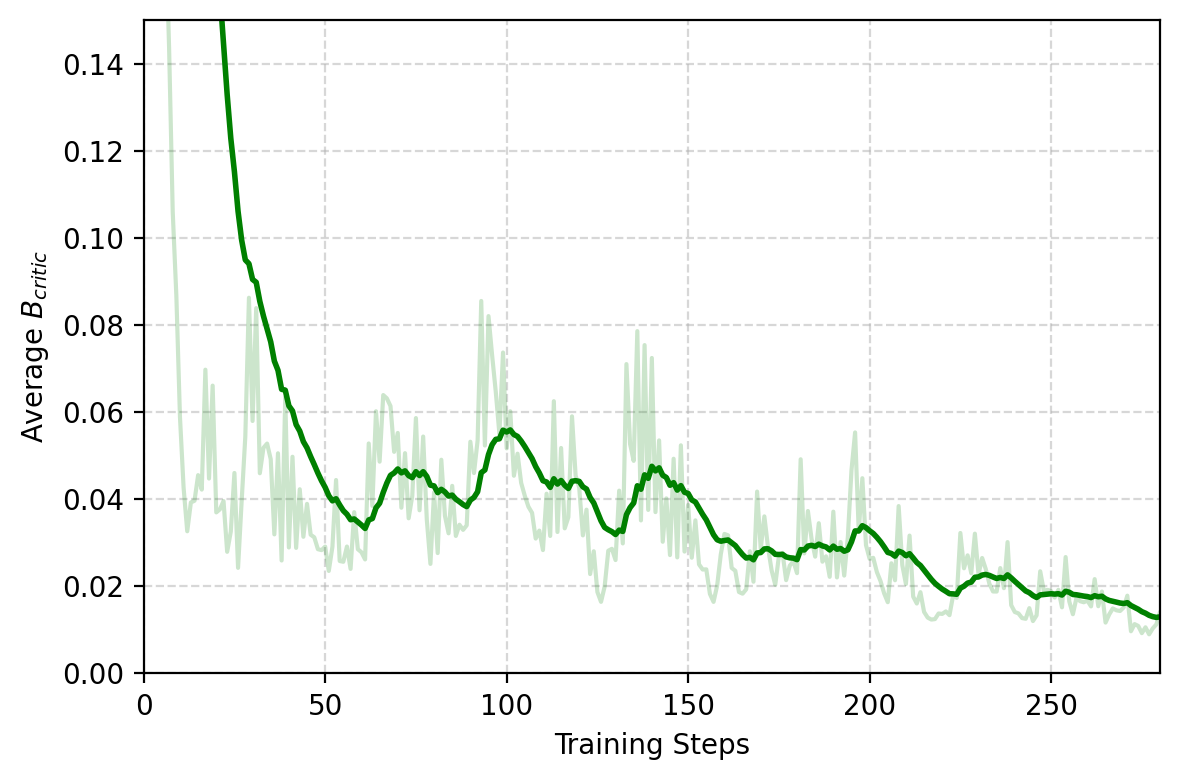}
        \caption{}
    \end{subfigure}
    \hfill
    \begin{subfigure}[c]{0.35\columnwidth}
        \centering
        \vspace{-1em}
        \includegraphics[width=0.9\linewidth]{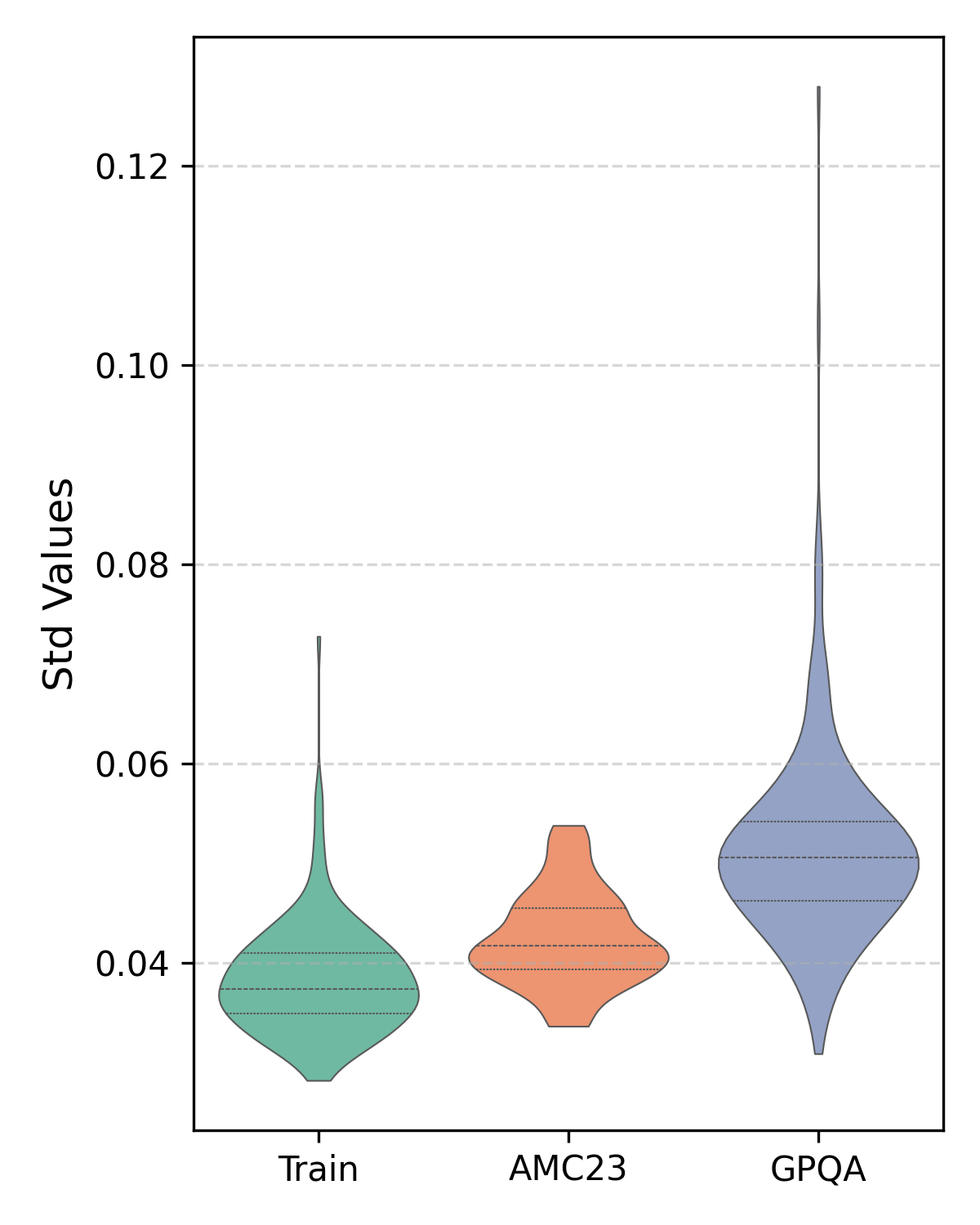}
        \caption{}
    \end{subfigure}
    \hfill
    \caption{(a) The average $B_{\text{critic}}$ over training steps. (b) Distribution of the standard deviation of value heads across prompts from different datasets.}
    \label{fig:MHA_dynamic}
\end{figure}

In sub-figure (b) of Figure \ref{fig:MHA_dynamic}, we present a cross-dataset analysis by calculating the average standard deviation of the value estimates across different questions. Specifically, we evaluate three datasets: the training set (DAPO-17K), the in-domain validation set (AMC23), and the out-of-domain validation set GPQA \citep{rein2023gpqa}. We observe that the training set exhibits a smaller standard deviation compared to both the in-domain and out-of-domain validation sets. This pattern aligns with the intuition that multi-head critics tend to show stronger disagreement on data that is less frequently encountered during training.

\textbf{Analysis of sub-sample fraction $\zeta$ during critic update}
Additionally, we examine the sensitivity of the critic update to the hyperparameter $\zeta$ (sub-sample fraction). We vary $\zeta$ under two configurations—critics with 16 heads and with 4 heads—and compare $\zeta\in\{0.5, 1\}$. As shown in Table~\ref{tab:mask_frac}, while a larger number of heads benefits from a larger sub-sample fraction, the overall performance is stable across settings. The model demonstrates robustness to the masking fraction $\zeta$, achieving similar results for both values tested (0.5 and 1.0).

\begin{table*}[h!]
    \centering
    \resizebox{0.9\textwidth}{!}{%
    \begin{tabular}{l c cc cc cc c}
        \toprule
        \multirow{2}{*}{\bf Model} 
        & \multicolumn{1}{c}{\bf MATH} 
        & \multicolumn{2}{c}{\bf AMC23} 
        & \multicolumn{2}{c}{\bf AIME24} 
        & \multicolumn{2}{c}{\bf AIME25} 
        & \multicolumn{1}{c}{\bf Avg} 
        \\
        & Avg@1 & Avg@16 & Pass@16 & Avg@16 & Pass@16 & Avg@16 & Pass@16\\
        \midrule
        \multicolumn{8}{c}{\textbf{\textit{Mask fraction}}} \\
        \textbf{16 Heads ; $\zeta=0.5$}   & 88.3 & 65.0 & 88.7 & 20.5 & 41.9 & 20.0 & 38.8 & 48.6 \\
        16 Heads ; $\zeta=1$   & 85.4 & 65.3 & 85.3 & 21.0 & 39.2 & 21.7 & 43.2 & 48.4 \\
        4 Heads ; $\zeta=0.5$    & 86.1 & 66.4 & 85.8 & 18.1 & 36.7 & 23.1 & 39.1 & 48.4 \\
        \textbf{4 Heads ; $\zeta=1$}    & 87.3 & 63.9 & 87.9 & 21.5 & 35.5 & 21.5 & 45.5 & 48.5 \\
        \bottomrule
    \end{tabular}
    }
    \caption{Ablation study on sub-sample fraction $\zeta$.}
    \label{tab:mask_frac}
\end{table*}

\section{Related Work}
\subsection{Reinforcement Learning (RL) for LLM reasoning}
Reinforcement Learning is a central technique for advancing the reasoning capabilities of LLMs. Initial approaches relied on reward models that provided either outcome-based supervision, focusing on the final answer \citep{cobbe2021training}, or process-based supervision, evaluating intermediate reasoning steps \citep{uesato2022solving}.  To navigate more complex problem spaces, these foundational reward strategies were often augmented with search algorithms such as MCTS \citep{feng2023alphazero, tian2024toward, chen2024alphamath, wang2024towards} and Q* \citep{wang2024q,wang2024litesearch}. More recently, RLVR \citep{lambert2024tulu} has emerged as a powerful alternative, demonstrating significant performance on complex reasoning tasks in mathematics and coding \citep{guo2025deepseek}. Consequently, a growing body of work seeks to apply RLVR to diverse domains, including multi-modal reasoning \citep{wang2025vl, li2025self}, logical reasoning \citep{zhou2025dissecting}, search engine use \citep{jin2025search, xiong2025rag}, and information extraction \citep{dai2025r1}.  Parallel efforts aim to improve upon the standard RLVR paradigm with techniques such as mixture-of-thought \citep{zheng2025learning}, self-evolving \citep{huang2025r}, parallel thinking~\citep{Zheng2025ParallelR1TP}. Despite these advances, persistent concerns remain regarding robustness \citep{dai2025breach, zhao2025one}, calibration \citep{shen2024thermometer}, and a lack of exploration evidenced by entropy collapse \citep{cui2025entropy, shen2025entropy}, highlighting the need for more principled training approaches.

\subsection{Efficient exploration}
Efficient exploration is a central challenge in Reinforcement Learning (RL), which aim to balance between exploration and exploitation \citep{SuttonRL, weng2020exploration, amin2021survey}. Many foundational approaches are heuristic-based, such as Gaussian noise \citep{lillicrap2015continuous} or the $\epsilon$-greedy method \citep{SuttonRL}. Entropy regularization is a more principled heuristic, which encourages the policy to be more stochastic. While simple to implement, these methods are often undirected—they promote pure randomness. Consequently, they can be suboptimal \citep{dann2022guarantees} with no significant gains in complex Deep RL \citep{andrychowicz2021matters} or LLM training \citep{cui2025entropy, shen2025entropy}.

In contrast, a major class of methods incentivizes exploration by adding exploration bonus to guide the agent toward novel or uncertain parts of the environment. Count-based approaches like UCB \citep{lai1987adaptive}, LinUCB \citep{li2010contextual}, and LSVI-UCB \citep{jin2020provably} use pseudo-counts of state-action visitations to encourage exploring rarely visited areas, achieving near-optimal theoretical guarantees in bandits and linear MDPs. Similarly, prediction-based methods such as ICM \citep{pathak2017curiosity} and RND \citep{burda2018exploration} use the error from a predictive model as a bonus, rewarding the agent for reaching states that are difficult to predict. Applying these guided exploration principles is a growing field in LLM. For instance, \citet{bai2025online} incorporate a count-based bonus into the RLHF process by introducing a coin flipping module. \citet{gao2025navigate} draws inspiration from RND by adding an auxiliary noise prediction network. However, both methods rely on expressive representations of long COT trajectories and introduce additional modules, which complicates the training framework. In contrast, \textbf{CDE} uses intrinsic curiosity signals from the actor and critics, requiring only minimal modifications to the framework and yielding efficient exploration both theoretically and empirically.

\section{Conclusion and Future Work}
We have presented Curiosity-Driven Exploration, an efficient technique that enhances agent learning by incorporating curiosity signals from both the actor and the critic. Our approach is notably lightweight, demanding only minor modifications to the original training architecture. Its effectiveness is demonstrated by consistent accuracy improvements over strong baselines on a suite of challenging mathematical reasoning benchmarks, with these empirical results strongly corroborating our underlying theoretical framework and intuition.

The \textbf{calibration collapse} revealed in our analysis aligns with recent findings on the root causes of LLM hallucination \citep{openai2025why}, pointing to a promising avenue for future work. We hypothesize that the underlying source of this collapse is the reward design of RLVR training. Specifically, RLVR with outcome reward prioritizes correct final outcomes at the expense of rigorous intermediate reasoning. Our experiments shed light on this direction by demonstrating that an alternative multi-perspective reward design (e.g., the PPL bonus) can be valuable for guiding the RLVR process more effectively.

\bibliography{colm2024_conference}
\bibliographystyle{colm2024_conference}

\newpage
\appendix
\section{Training Details}
\label{sec:detail}
We use \texttt{verl} as the training framework\footnote{\url{https://github.com/volcengine/verl}}.
Configurations for training CDE and baseline models are listed in Table~\ref{tab:training-detail}.

\begin{table}[h!]
    \centering
    \begin{subtable}[t]{0.49\linewidth}
        \centering
        \resizebox{\linewidth}{!}{
        \begin{tabular}{lcc}
            \toprule
            Config & \textbf{GRPO} & \textbf{PPO}  \\
            \midrule
            actor-lr               & 1e-6 & 1e-6 \\
            critic-lr              & - & 1e-5 \\
            critic-warmup          & - & 10 \\
            kl\_coef               & 0.0  & 0.0  \\
            max\_prompt\_length    & 2K   & 2K   \\
            max\_response\_length  & 3K   & 3K   \\
            train\_batch\_size     & 256  & 512  \\
            ppo\_mini\_batch\_size & 256  & 256  \\
            clip\_ratio            & 0.20 & 0.20 \\
            sample temperature     & 1.0  & 1.0  \\
            rollout.n              & 8    & 4    \\
            total\_training\_steps & 300  & 300  \\
            \bottomrule
        \end{tabular}}
        \caption{}
    \end{subtable}\hfill
    \begin{subtable}[t]{0.49\linewidth}
        \centering
        \resizebox{\linewidth}{!}{
        \begin{tabular}{lccc}
            \toprule
            Config & \textbf{PPL} & \textbf{2,4 Heads} & \textbf{8,16 Heads} \\
            \midrule
            $\kappa$               & 3 & 3 & 3 \\
            $\alpha$              & 1 & 0.5 & 0.5 \\
            $\omega_t$          & Staircase & No decay & No decay \\
            $\zeta$               & -  & 1 & 0.5  \\
            \bottomrule
        \end{tabular}}
        \caption{}
    \end{subtable}
    \caption{(a) Baseline training configurations. The \textbf{GRPO} setup is shared across all GRPO-based methods (e.g., ``Qwen3-4B-Base-GRPO'' and ``w/PPL bonus'' in Table \ref{tab:main}); likewise, the \textbf{PPO} setup is shared across all PPO-based methods. (b) \textbf{CDE}-specific configurations. The \textbf{PPL} settings are identical for both the GRPO ``w/PPL bonus'' and PPO ``w/PPL bonus'' variants.}
    \label{tab:training-detail}
\end{table}

\section{Prompt}
\label{sec:prompt}

\begin{tcolorbox}[
    colframe=pink!80!black,
    colback=pink!10!white,
    fontupper=\small,
    boxrule=0.4mm,
    left=2mm, right=2mm, top=0.5mm, bottom=0.5mm,
    arc=2mm,
    enhanced
]
Solve the following math problem step by step. The last line of your response should be of the form Answer: \$Answer (without quotes) where \$Answer is the answer to the problem.

\{Problem\}

Remember to put your answer on its own line after ``Answer:''.
\end{tcolorbox}

\noindent\begin{minipage}{\linewidth}
\captionof{figure}{The prompt for RLVR training.}
\end{minipage}

\section{Details on Hash-based pseudo count}
\label{sec:hash}
The core idea is to map a full prompt–response trajectory to a compact hash that serves as a pseudo-state for exploration. Given a prompt–response pair $(q,o)$ with tokenized sequences $q=\{q_1,\ldots,q_D\}$ and $o=\{o_1,\ldots,o_T\}$, let the model produce last-layer hidden states $\mathbb{h}=\{h_1,\ldots,h_{D+T}\}$, $h_i\in\mathbb{R}^{d}$. We form a trajectory embedding $h_{q,o}\in\mathbb{R}^d$ from $\mathbb{h}$ via one of: (i) $h_{D+T}$; (ii) $h_{D+T-1}$; or (iii) mean pooling $\frac{1}{D+T}\sum_{i=1}^{D+T}h_i$. With a random projection matrix $A\in\mathbb{R}^{k\times d}$ (rows drawn i.i.d.\ from $\mathcal{N}(0,I)$ or Rademacher), we compute a $k$-bit SimHash code
\begin{equation*}
    \phi(q,o)=\operatorname{sign}\!\big(Ah_{q,o}\big)\in\{-1,+1\}^k,
\end{equation*}
and map it to a bucket index $b=\mathrm{bucket}(\phi)\in\{0,\ldots,2^k-1\}$. Let $n(b)$ be the visitation count of bucket $b$. We apply intrinsic reward shaping to encourage rarely visited trajectories:
\begin{equation*}
    \tilde r(q,o)=r(q,o)+\omega_t\frac{\beta_t}{\sqrt{n(b)}},
\end{equation*}
where $\beta_t$ is the weight for exploration bonus a. This yields an efficient, matrix-inversion–free intrinsic bonus that scales linearly in $k d$ per sample, following \citet{tang2017exploration}.


\begin{algorithm}[t]
\caption{Count-based exploration for RLVR through SimHash}
\label{alg:simple-simhash}
\DontPrintSemicolon
\SetKwInOut{Input}{Inputs}
\Input{Policy $\pi_\theta$; aggregator $g$; random projection matrix $A\in\mathbb{R}^{k\times d}$; hash counts $n[\cdot]\leftarrow 0$; weights $\beta_t$.}
\For{each training iteration $t=1,2,\dots$}{
  Sample prompts $q$ and generate responses $o\sim \pi_\theta(\cdot\mid q)$\;
  \For{each $(q,o)$ in the batch}{
    Obtain last-layer token states $\{h_i\}_{i=1}^{|q|+|o|}$; set $h \leftarrow g(\{h_i\})$\;
    $c \leftarrow \mathrm{sign}(A h)$,\quad $b \leftarrow \mathrm{bucket}(c)$\;
    $n[b] \leftarrow n[b] + 1$\;
    $\tilde r(q,o) \leftarrow r(q,o) + \beta_t/\sqrt{n[b]}$\;
  }
  Update $\pi_\theta$.
}
\end{algorithm}

\section{Proof for Calibration Theorem}
\label{sec:calibration}
Define $\Tilde{r}_t(q,o) = r(q,o) + b_t(q,o)$ where $b_t(q,o)  = \omega\min\{\kappa |r(q,o)| , -\frac{\alpha}{T_o} \log \pi_t(o|q)\}$  is a bonus function where $T_o$ is the length of response $o$.  Note that $\omega$ is a redundant variable in theory because we can write $b_t(q,o)  = \min\{\kappa' |r(q,o)| , -\frac{\alpha}{T_o} \log \pi_t(o|q)\}$ with $\kappa' = \omega \kappa$ and $\alpha' = \omega \alpha$. Given that $r(x,y) \in \{1,-1\}$, it suffices to consider $b_t(q,o) = \min\{\kappa , - \frac{\alpha}{T_o} \log \pi_t(o|q)\}$. Thus, as long as we use $\kappa < 1$, we have $\text{sign}(\Tilde{r}_t(q,o)) =  \text{sign}(r(q,o))$. The introduce of bonus does not change the sign of the original correctness reward.
 
Consider single step policy optimization
\begin{align*}
    \pi_{t+1}(\cdot|q) = \arg\max_{\pi}\left\{ \sum_{o} \pi(o|q)\tilde{r}_t(q,o) - \frac{1}{\eta}\text{KL}\left(\pi(\cdot|q) \| \pi_t(\cdot|q)\right)\right\},
\end{align*}
which has closed-form solution 
\begin{align*}
     \pi_{t+1}(o|q) =  \frac{\pi_t(o|q) \exp\left(\eta \Tilde{r}_t(q,o)\right)}{\sum_{o'}\pi_t(o'|q) \exp\left(\eta\Tilde{r}_t(q,o')\right) }.
\end{align*}
For any question $q$ and response $o$. Define $Z(q) = \sum_{o'}\pi_t(o'|q) \exp\left(\eta\Tilde{r}_t(q,o')\right)$, we have
\begin{align*}
    \log\pi_{t+1}(o|q) = \log\pi_t(o|q) +  \eta \Tilde{r}_t(q,o) -\log\left(Z(q)\right).
\end{align*}
Define $\Delta_t(o|q) = \log\pi_{t+1}(o|q) - \log\pi_{t}(o|q)$ as the change of likelihood of response $o$ under question $q$ at update step $t$. For two correct response $o^+_1$ and $o^+_2$ with length $T_{o^+_1}$ and $T_{o^+_2}$, and  $-\frac{\alpha}{T_{o^+_1}} \log \pi_t(o^+_1|q) \ge  - \frac{\alpha}{T_{o^+_2}} \log \pi_t(o^+_2|q)$ (i.e. $o_1^+$ has larger perplexity), we have
\begin{align*}
&\Delta_t(o^+_1|q) - \Delta_t(o^+_2|q)
\\&= \Tilde{r}_t(q,o^+_1) -   \Tilde{r}_t(q,o^+_2) 
\\&=  b_t(q,o^+_1) -  b_t(q,o^+_2)
\\&= \min\{\kappa ,  -\frac{\alpha}{T_{o^+_1}} \log \pi_t(o^+_1|q)\} - \min\{\kappa , - \frac{\alpha}{T_{o^+_2}} \log \pi_t(o^+_2|q)\}
\\&\ge 0
\end{align*}

Similarly, for  two incorrect response $o^-_1$ and $o^-_2$ with $-\frac{\alpha}{T_{o^-_1}} \log \pi_t(o^-_1|q) \ge  - \frac{\alpha}{T_{o^-_2}} \log \pi_t(o^-_2|q)$ (i.e. $o_1^-$ has larger perplexity), we have $\Delta_t(o^-_1|q) - \Delta_t(o^-_2|q) \ge 0$.

Specifically, given a question $q$, for any response $(o_1, o_2)$ that has the same correctness label and $-\frac{\alpha}{T_{o_1}} \log \pi_t(o_1|q) \ge  - \frac{\alpha}{T_{o_2}} \log \pi_t(o_2|q)$,  we have
\begin{itemize}
    \item  If $\Tilde{r}_t(q,o_1) \ge \frac{1}{\eta}\log\left(Z(q)\right)$ and $\Tilde{r}_t(q,o_2) \ge \frac{1}{\eta}\log\left(Z(q)\right)$ , then $\Delta_t(o_1|q) \ge 0$ and $\Delta_t(o_2|q) \ge  0$ but $o_1$ has more likelihood increase.
    \item  If $\Tilde{r}_t(q,o_1) \ge \frac{1}{\eta}\log\left(Z(q)\right)$ and $\Tilde{r}_t(q,o_2) < \frac{1}{\eta}\log\left(Z(q)\right)$ , then $\Delta_t(o_1|q) \ge 0$ and $\Delta_t(o_2|q) < 0$ where $o_1$'s likelihood increase but $o_2$'s likelihood decrease.
    \item  If $\Tilde{r}_t(q,o_1) < \frac{1}{\eta}\log\left(Z(q)\right)$ and $\Tilde{r}_t(q,o_2) < \frac{1}{\eta}\log\left(Z(q)\right)$ , then $\Delta_t(o_1|q) < 0$ and $\Delta_t(o_1|q) < 0$ but $o_1$ has less likelihood decrease.
    \item  It is impossible that $\Tilde{r}_t(q,o_1) < \frac{1}{\eta}\log\left(Z(q)\right)$ and $\Tilde{r}_t(q,o_2) \ge \frac{1}{\eta}\log\left(Z(q)\right)$ given that  $(o_1, o_2)$ has the same correctness label and $-\frac{\alpha}{T_{o_1}} \log \pi_t(o_1|q) \ge  - \frac{\alpha}{T_{o_2}} \log \pi_t(o_2|q)$.
\end{itemize}

\section{Proof for Consistency of Multi-head Critic Bonus}
\label{sec:LinMDP}

\textbf{Linear MDP and Assumptions}

\begin{assumption}[Linear MDP]
We consider finite horizon $\mathcal{M} = (\mathcal{S}, \mathcal{A}, R, P, H)$ with horizon $H$, state space $\mathcal{S}$, action space $\mathcal{A}$, reward function $R: \mathcal{S} \times \mathcal{A} \rightarrow \mathbb{R}$, and transition $P: \mathcal{S} \times \mathcal{A} \rightarrow \mathcal{S}$ such that there exists a known feature $\phi \in \mathbb{R}^d$ and unknown features $\theta, \psi \in \mathbb{R}^d$ to ensure
\begin{align*}
    R(s,a) = \phi(s,a)^\top \theta \quad \quad P(s'|s,a) = \phi(s,a)^\top \psi(s').
\end{align*}
Without loss of generality, we assume $\|\phi(s,a)\| \le 1$ for all $(s,a)$, and $\|\psi(s')\| \le \sqrt{d}, \|\theta\|_2 \le \sqrt{d}$.
\label{ass:linearMDP}
\end{assumption}

\begin{lemma}[Proposition~2.3 in \citet{jin2020provably}]
    For linear MDPs that satisfy Assumption \ref{ass:linearMDP},  there exists $w_h^\star\in\R^d$ such that 
    \begin{equation*}
    Q_h^\pi(s,a) := \E\Big[\sum_{t=h}^H  r_t \big| s_h=s,a_h=a \Big] = \phi(s,a)^\top w_h^\star.
    \end{equation*}
    \label{ass:linear-Q}
\end{lemma}
The linearity of $Q$-functions enables using regression technique to solve it. Consider a dataset with $n$ observations $\mathcal{D} = \{s_{i,h}, a_{i,h}, G_{i,h}\}_{i=1}^n$ where $G_{i,h}$ is the Monte-Carlo return. Let $\phi_{i,h} = \phi(s_{i,h}, a_{i,h})$ and denote the regression noise as $\varepsilon_{i,h} = G_{i,h} - \phi_{i,h}^\top w_h^\star$. We impose the following assumptions.

\begin{itemize}
\item[(A1)] $\E[\varepsilon_{i,h}\mid \phi_{i,h}] = 0$ and $\{(\varepsilon_{i,h})\}_{i=1}^n$ are i.i.d. $\sigma^2$–sub-Gaussian for each fixed $h$;
\item[(A2)] $\frac{1}{n}\sum_{i=1}^n \phi_{i,h}\phi_{i,h}^\top \xrightarrow{\mathbb{P}} \Sigma_t \succ 0$
\end{itemize}

\cite{jin2020provably} shows that doing value iteration on optimistically estimated Q function can achieve near-optimal regret for linear MDP, where the optimistic Q function is the combination of linear regression estimation and exploration bonus $b_{n,h} = \beta \sqrt{\phi_{n,h}^\top \Lambda^{-1}_{n,h} \phi_{n,h}}$, where $\Lambda_{n,h} = \lambda I + \sum_{i=1}^n \phi_{i,h}\phi_{i,h}^\top$ and $\beta$ is some constant. Below we will formally connect our bootstrapped bonus with this term.

\textbf{Formulation of the bootstrap multi-head critic}

We accommodate the bootstrap multi-head into the linear-MDP setting. For any time step $h$, we sample $K$ mini-batches $\{S_k \subset[n]\}_{k=1}^K$ of size $m=\zeta n$ uniformly without replacement from $\mathcal{D}$ and construct the ridge estimator as follows
\begin{equation*}
    \widehat w_{n,h}^{(k)} = \arg\min_{w}\sum_{r\in S_k}(G_{r,h}-\phi_{r, h}^\top w)^2+ \zeta \lambda\|w\|^2.
\end{equation*}
For any feature $\phi\in\R^d$, we define the bootstrap multi-head bonus as
\begin{equation*}
b^{\text{boot}}_{h, K}(\phi) = \text{std}\left( \big\{ \phi^\top \widehat w_{n,h}^{(k)} \big| 1 \leq k \leq K \big\} \right).
\end{equation*}

\textbf{Elliptical (``count-based'') bonus in \citep{jin2020provably}.}
The ridge estimator is constructed using all data across $n$ trajectories as follows
\begin{equation*}
    \widehat w_{n,h} = \arg\min_{w}\sum_{i=1}^n(G_{i,h}-\phi_{i, h}^\top w)^2+ \zeta \lambda\|w\|^2.
\end{equation*}
For any query feature $\phi\in\R^d$, the bonus term is $b^{\text{cnt}}_{h}(\phi) = \sqrt{\phi^\top \Lambda_{n,h}^{-1}\phi}$.

\paragraph{Formal Version and Proof of Theorem \ref{thm:consistency}}
\begin{theorem}
Under Assumption \ref{ass:linearMDP} and assumptions (A1)–(A2), for any fixed time-step $h$ and query $\phi\in\R^d$,
\begin{equation*}
b^{\text{boot}}_{h,K}(\phi)
\xrightarrow[\ K\to\infty,\ n\to\infty\ ]{\ \mathbb{P}\ }
\beta \sqrt{\phi^\top \Lambda_{n,h}^{-1}\phi},
\end{equation*}
where $\beta$ is some constant.
\end{theorem}

\begin{proof}
For any time-step $h$ and $S_k \subset [n]$, we have the explicit solution of the ridge regression
\begin{equation*}
    \widehat{w}_{n,h} = \Lambda_{n,h}^{-1} \sum_{i=1}^{n} \phi_{i,h}G_{i,h}.
\end{equation*}
Conditioning on $X_h=[\phi_{1,h}^\top;\ldots;\phi_{n,h}^\top]$, the conditional variance of the estimator is
\begin{equation*}
\Var\big(\phi^\top \widehat w_{n,h}\mid X_h\big)
= \sigma^2 \phi^\top\Big(\Lambda_{n,h}^{-1}-\lambda \Lambda_{n,h}^{-2}\Big)\phi.
\end{equation*} 
From Assumption (A2), we have  $\|\Lambda_{n,h}^{-1}\|_{\mathrm{op}} = O_p(1/n)$, therefore
\begin{equation*}
\phi^\top \Lambda_{n,h}^{-1}\phi \leq \|\phi\|^{2}\,\|\Lambda_{n,h}^{-1}\|_{\mathrm{op}}
= O_p(1/n) ~~ \text{and} ~~ \phi^\top \Lambda_{n,h}^{-2}\phi = O_p(1/n^2),
\end{equation*}
and
\begin{equation*}
    n\phi^\top\Big(\Lambda_{n,h}^{-1}-\lambda \Lambda_{n,h}^{-2}\Big)\phi - n\phi^\top\Lambda_{n,h}^{-1}\phi = - n \lambda \phi^\top \Lambda_{n,h}^{-2}\phi \xrightarrow{ \PP } 0.
\end{equation*}
Therefore, we have
\begin{equation*}
    \phi^\top\Big(\Lambda_{n,h}^{-1}-\lambda \Lambda_{n,h}^{-2}\Big)\phi \xrightarrow{ \PP } \phi^\top\Lambda_{n,h}^{-1}\phi.
\end{equation*}

Before moving to $b^{\text{boot}}_{h,K}(\phi)$, we define the following quantities
\begin{equation*}
    \Delta\Sigma = \frac{1}{\zeta}\sum_{r\in S_k}\phi_{r,h}\phi_{r,h}^\top - \sum_{i=1}^n \phi_{i,h}\phi_{i,h}^\top, \quad b = \sum_{i=1}^{n} \phi_{i,h}G_{i,h}, \quad  b_s = \frac{1}{\zeta} \sum_{r\in S_k} \phi_{r,h}G_{r,h}, \quad  \Delta b = b_s - b.
\end{equation*}
Since $\Sigma_t \succ 0$,  matrix Bernstein for sampling without replacement yields $\|\Delta\Sigma\|_{\mathrm{op}} = O_p(\sqrt{n})$.
Use the expansion
\[
(\Lambda_{n,h} + \Delta\Sigma)^{-1}
= \Lambda_{n,h}^{-1} - \Lambda_{n,h}^{-1}\Delta\Sigma \Lambda_{n,h}^{-1} + R_\Sigma,
\quad
\|R_\Sigma\|_{\mathrm{op}} = O_{p}(\|\Lambda_{n,h}^{-1}\|_{\op}^3\|\Delta\Sigma\|_{\op}^2) = O_p(1/n^2).
\]

The $k$-th bootstrap ridge solution is
\begin{equation*}
    \widehat w^{(k)}_{n,h}=(\Lambda_{n,h}+\Delta\Sigma)^{-1} b_s.
\end{equation*}
Subtracting $\widehat w_{n,h}=\Lambda_{n,h}^{-1}b$ and inserting the expansion,
\begin{equation*}
\widehat{w}^{(k)}_{n,h}-\widehat w_{n,h} = \underbrace{\Lambda_{n,h}^{-1}\Delta b - \Lambda_{n,h}^{-1}\Delta\Sigma\,\widehat w_{n,h}}_{\text{first order}} + \underbrace{\big(-\Lambda_{n,h}^{-1}\Delta\Sigma\Lambda_{n,h}^{-1}\Delta b + R_\Sigma b_s \big)}_{=:r_n}.
\end{equation*}
Since $G_{i,h} = \phi_{i,h}^\top w_h + \epsilon_{i,h}$, for any $\phi$ we have 
\begin{equation}
\phi^\top(\widehat{w}^{(k)}_{n,h}-\widehat w_{n,h}) = \phi^\top \Lambda_{n,h}^{-1}\Big(\frac{1}{\zeta}\sum_{r\in S_k}\phi_{r,h}\epsilon_{r,h} - \sum_{i=1}^{n}\phi_{i,h}\epsilon_{i,h}\Big) + \phi^\top \Lambda_{n,h}^{-1}\Delta\Sigma(w_h^\star - \widehat w_{n,h}) + \phi^\top r_n .
\label{eq:decomp}
\end{equation}
From standard results for ridge regression, we have $\|w_h^\star - \widehat w_{n,h}\|_{2} = O_{\PP}(1/\sqrt{n})$, thus we have the second term $\phi^\top \Lambda_{n,h}^{-1}\Delta\Sigma(w_h^\star - \widehat w_{n,h}) = O_{\PP}(1/n)$. Similarly, for the last term we have 
\begin{equation*}
    \phi^\top r_n \leq \|\phi\| \left(\|\Lambda_{n,h}^{-1}\|^2_{\mathrm{op}} \|\Delta\Sigma\|_{\mathrm{op}} \|\Delta b\|_{\mathrm{op}} + \|\Delta\Sigma\|_{\mathrm{op}}\|b_s\|_{\mathrm{op}} \right) = O_{\PP}(1/n).
\end{equation*}
Therefore, both terms are negligible at the
$\sqrt{\cdot}$ scale. Condition on $(X_h,\{\epsilon_{i,h}\}_{i=1}^n)$ the only randomness comes from $S$. By finite-population sampling theory,
\begin{equation*}
     \Var^* \Big( \frac{1}{\zeta} \sum_{r\in S} \phi_{r,h}\epsilon_{r,h} \Big) = \frac{1-\zeta}{\zeta}\sum_{i=1}^n \phi_{i,h} \phi_{i,h}^\top \sigma^2 .
\end{equation*}
Therefore,
\begin{align*}
    \Var^* \Big( \phi^\top(\widehat{w}^{(k)}_{n,h}-\widehat w_{n,h}) \Big) &= \frac{1-\zeta}{\zeta} \sigma^2 \phi^\top \Lambda_{n,h}^{-1} \Big(\sum_{i=1}^n \phi_{i,h} \phi_{i,h}^\top\Big) \Lambda_{n,h}^{-1} \phi + o_{\PP}\big(1/n\big) \\
    &= \frac{1-\zeta}{\zeta} \sigma^2  \phi^\top\Big(\Lambda_{n,h}^{-1}-\lambda \Lambda_{n,h}^{-2}\Big)\phi + o_{\PP}\big(1/n\big) \\
    &= \frac{1-\zeta}{\zeta} \sigma^2 \phi^\top\Lambda_{n,h}^{-1}\phi + o_{\PP}\big(1/n\big)
\end{align*}

Finally, by the conditional strong law of large numbers, we have
\begin{equation*}
    b^{\text{boot}}_{h, K}(\phi) = \text{std}\left( \big\{ \phi^\top \widehat w_{n,h}^{(k)} \big| 1 \leq k \leq K \big\} \right) \to_{\text{a.s.}} \sqrt{\Var^{*}(\phi^\top \widehat w_{n,h}^{(k)})} \xrightarrow{\mathbb{P}} \sqrt{\frac{1-\zeta}{\zeta}}\sigma \sqrt{\phi^\top \Lambda_{n,h}^{-1}\phi}.
\end{equation*}
\end{proof}

\end{document}